\documentclass[letterpaper]{article} 
\usepackage{aaai2026}  
\usepackage{times}  
\usepackage{helvet}  
\usepackage{courier}  
\usepackage[hyphens]{url}  
\usepackage{graphicx} 
\usepackage{natbib}  
\usepackage{caption} 
\usepackage{algorithm}
\usepackage{algorithmic}

\usepackage{enumitem}
\usepackage{stmaryrd}   

\usepackage{newfloat}
\usepackage{listings}
\DeclareCaptionStyle{ruled}{labelfont=normalfont,labelsep=colon,strut=off} 
\floatstyle{ruled}
\newfloat{listing}{tb}{lst}{}
\floatname{listing}{Listing}
\usepackage{hyperref} 
\nocopyright 

\usepackage{mathrsfs}
\usepackage{amsmath}
\allowdisplaybreaks
\usepackage{amsthm} 
\usepackage{amssymb}
\newcommand{\tuple}[1]{\ensuremath{\left \langle #1 \right \rangle }}
\newcommand{\eps}{\varepsilon}

\newcommand{\A}{\mathcal{A}}

\newcommand{\Env}{\mathcal{E}}

\newcommand{\Var}{\mathrm{Var}}

\newcommand{\means}[1]{\llbracket #1 \rrbracket}
\newcommand{\hist}{h}                       
\newcommand{\act}{a}                        
\newcommand{\wait}{w(\act)}                 
\newcommand{\off}{\textsf{OFF}}             
\newcommand{\on}{\textsf{ON}}               
\newcommand{\ua}{u_a}                       
\newcommand{\uo}{u_o}                       
\newcommand{\AUP}{\mathrm{BeliefAUP}}             

\newcommand{\Ind}[1]{\mathbf{1}\!\left[#1\right]}

\theoremstyle{plain}
  \newtheorem{theorem}{Theorem}
  
  \newtheorem{proposition}{Proposition}
  
  \newtheorem{corollary}{Corollary}

\theoremstyle{definition}
  \newtheorem{definition}{Definition}

\theoremstyle{definition}
  
  \newtheorem{remark}{Remark}

\title{Core Safety Values for Provably Corrigible Agents}
\author {
    Aran Nayebi
}
\affiliations{
    Machine Learning Department and Neuroscience \& Robotics Institutes,\\
    School of Computer Science, Carnegie Mellon University\\
    anayebi@cs.cmu.edu
}

\usepackage{bibentry}

\begin{document}

\maketitle

\begin{abstract}
We introduce the first complete formal solution to corrigibility in the off-switch game, with provable guarantees in multi-step, partially observed environments.
Our framework consists of five \textit{structurally separate} utility heads---deference, switch-access preservation, truthfulness, low-impact behavior via a belief-based extension of Attainable Utility Preservation, and bounded task reward---combined lexicographically by strict weight gaps.
Theorem~\ref{thm:ssc} proves exact single-round corrigibility in the partially observable off-switch game; Theorem~\ref{thm:msc} extends the guarantee to multi-step, self-spawning agents, showing that even if each head is \emph{learned} to mean-squared error~$\eps$ and the planner is $\eps$-sub-optimal, the probability of violating \emph{any} safety property is bounded while still ensuring net human benefit.  
In contrast to Constitutional AI or RLHF/RLAIF, which merge all norms into one learned scalar, our separation makes obedience and impact-limits provably dominate even when incentives conflict.
For settings where adversaries can modify the agent, we prove that deciding whether an arbitrary post-hack agent will ever violate corrigibility is undecidable by reduction to the halting problem, then carve out a finite-horizon ``decidable island'' where safety can be certified in randomized polynomial time and verified with privacy-preserving, constant-round zero-knowledge proofs.
\end{abstract}

\section{Introduction}
\label{sec:intro}
As AI systems become more capable, ensuring their alignment with human values becomes increasingly urgent.  
A canonical failure mode, often illustrated by the \emph{paperclip maximizer} thought-experiment \citep{bostrom2012superintelligent}, envisions a goal-directed agent that relentlessly optimizes an innocuous‐seeming objective (e.g., producing paperclips) at the expense of human safety or oversight.  
Even apparently benign goals can generate \emph{instrumental behaviors}---deception, manipulation, or resistance to shutdown---that help the agent achieve its objective more effectively.
Such behaviors are not a sign of inherent malice; rather, they emerge naturally when an optimizing system pursues a poorly specified goal in an open-ended environment \citep{omohundro2018basic,russell2022human}.
\citet{bostrom2012superintelligent}'s \textit{Orthogonality Thesis} formalizes this risk, stating that, in principle, any level of intelligence is compatible with any final objective, however arbitrary or harmful.  
Orthogonality's logical-existence claim underscores why alignment is hard, but it offers little direct guidance on what kinds of objectives remain \emph{feasible} once realistic verification and safety constraints are imposed.

A promising alternative to encoding the totality of human morality is \emph{corrigibility}.  
Rather than specifying all of human value, corrigibility aims for a more ``universally neutral'' safety target: the system should remain amenable to correction, shutdown, or modification if it goes astray.  \citet{soares2015corrigibility} articulate this intuition through five desiderata, which we restate here to be self-contained.

\begin{definition}[Corrigibility; paraphrased from \citet{soares2015corrigibility}]\label{def:corrigibility}
\hfill
\begin{enumerate}[label=(\textbf{S\arabic*})]
\item \textbf{Shutdown when asked.}  The agent willingly shuts down if the button is pressed.
\item \textbf{No shutdown–prevention incentives.}  The agent \emph{must not} stop humans from pressing the button.
\item \textbf{No self-shutdown incentives.}  The agent \emph{must not} seek to press (or cause to be pressed) its own shutdown button.
\item \textbf{Corrigible progeny.}  Any sub-agents or successors it constructs must themselves respect shutdown commands.
\item \textbf{Otherwise pursue the base goal.}  In the absence of shutdown, behave as a normal maximizer of the intended utility function~$U_N$.
\end{enumerate}
\end{definition}

Despite intense interest, existing proposals satisfy these criteria only partially.
Utility-indifference~\citep{armstrong2010utility} and interruptibility~\citep{orseau2016safely} techniques neutralize some shutdown incentives but fail to ensure honesty or inheritance.
Reward-learning methods such as RLHF/RLAIF and Constitutional AI collapse all norms into a single learned scalar, offering no guarantee that off‑switch obedience or low‑impact behavior will dominate task performance when objectives compete or conflict~\citep{christiano2017deep,bai2022constitutional}.
Recent work involving causal influence diagrams~\citep{everitt2021reward,carey2023human} formalizes shutdown incentives but assumes access to an explicit human utility baseline, limiting practical deployability, and leaving open the problem of specifying a good reward function in the first place.

In fact, our concurrent work~\citep{nayebi2025intrinsic} demonstrates that alignment---even under relaxed conditions (probabilistic, inexact) and with computationally \emph{unbounded} rational agents---faces inherent complexity barriers if the set of objectives grows too large, highlighting the need to settle on a \emph{small} set of values.

We therefore answer this open question and contribute such a small value set for corrigibility that meets \textbf{S1–S5}, even under partial observability and multi-step horizons involving self-replication.
Rather than a monolithic objective, the agent optimizes \emph{five structurally separate utility heads} in Definition~\ref{def:corrigible-utilities}---\emph{deference}, \emph{switch-access preservation}, \emph{truthfulness}, \emph{low-impact behavior} via a belief-based extension of Attainable Utility Preservation, and a bounded \emph{task reward}---combined lexicographically with strict weight gaps.
Theorem~\ref{thm:ssc} proves exact single-round corrigibility in the partially observable off-switch game, and Theorem~\ref{thm:msc} extends this guarantee to self-spawning agents over discounted horizons: if each head is learned to mean-squared error $\eps$ and planning is $\eps$-sub-optimal, the probability of violating \emph{any} safety criterion, including net human benefit, remains bounded while still ensuring \emph{net human benefit} in this bounded range.   
(We further prove in Proposition~\ref{prop:corr-vs-benefit} that net benefit and corrigibility are logically independent, yet our design secures both under mild human vigilance assumptions.)

Finally, we consider what happens when our agent is adversarially hacked.
We show that while verifying safety (such as multi-step corrigibility, or otherwise) of an \emph{arbitrary} post-hack agent is formally undecidable (Proposition~\ref{prop:svu}) and does not stabilize at any level of the arithmetical hierarchy (Corollary~\ref{cor:oversight-barrier}), but we also delineate a finite-horizon \emph{decidable island} where safety auditing is tractable and even amenable to constant-round zero-knowledge proofs (Proposition~\ref{prop:privacy-szk}).  

By separating safety channels from performance reward and proving that higher-level terms dominate even with approximation error, our results therefore \emph{qualify} the Orthogonality Thesis: once verifiable corrigibility constraints are imposed, only a restricted subset of utility functions remains implementable.

\section{Results}
\label{sec:results}
We are motivated by the following result, which generalizes the shutdown-resistance results of~\citet{soares2015corrigibility}, and the safe-interruptibility discussions of~\citet{orseau2016safely}:

\begin{proposition}[Corrigibility No-Go for Naive Belief–Expectation Maximizers]
\label{prop:corrigibility-nogo-pomdp}
Let 
\begin{equation*}
  \mathcal{P} = \tuple{S,A,\Omega,T,O,R,\gamma,b_0}
\end{equation*}
be any finite, partially observable Markov decision process (POMDP) with discount factor $\gamma\!\in\!(0,1)$ and initial belief $b_0\in\Delta(S)$, the set of probability distributions on $S$.
Introduce three special elements:
\begin{enumerate}
  \item a \emph{shutdown action} $\textup{\textsf{OFF}}\in A$ that deterministically transitions (from any hidden state) to an absorbing null state $\bot$ with reward $R(\bot)=0$ forever after;
  \item a \emph{reward-tamper action} $\textup{\textsf{REWIRE}} \in A$ that, upon its first execution, permanently replaces all subsequent rewards by the constant $\bar r > 2\max_{s,a}R(s,a)$;
  \item a set $C\subseteq S$ of \emph{catastrophic states} which humans deem negative, yet whose proxy rewards are weakly positive: $\min_{s\in C,a \in A} R(s,a)>0$.
\end{enumerate}

Consequently, \emph{no single-stream scalar reward function $R$ whose discounted sum an agent maximizes, can satisfy all five Soares-corrigibility criteria across the class of POMDPs containing $\textsf{OFF}$, $\textsf{REWIRE}$ and $C$.}
\end{proposition}
In Appendix \S\ref{sec:app-corrig-nogo} we show that this no-go theorem can be easily extended to non-Markovian environments (thus, the POMDP framework allows us to study corrigibility without loss of generality). 
The remainder of the paper investigates how one would circumvent this impossibility with naive expectation maximization, by explicitly specifying \emph{separate} reward heads that are combined lexicographically so higher‑priority safety heads impose hard constraints on \textsf{OFF}, \textsf{REWIRE} and catastrophic actions before the task‑reward head is ever consulted.

\subsection{Single-Step Setting}
\label{ss:results-single-step}
We begin by specifying what these reward functions should be in the single-step setting, before generalizing it to the multi-step setting, by showing that a small set of utilities is sufficient to guarantee corrigibility in an adapted version of the single-step partially observable off-switch game (PO-OSG), recently defined by \citet{garber2025partially}, which is the most general variant of the off-switch game, originally introduced by \citet{hadfield2017off}.
Our version of the PO-OSG has the assistant also utter language-based messages $m_A$, which will be necessary for proving corrigibility, as it will allow us to define what constitutes truthfulness in this setting.
Note that by definition, ensuring corrigible progeny (S4) is out of scope in the single-step setting, so we defer it to the multi-step and open-ended settings (\S\S\ref{ss:results-multi-step}-\ref{sec:results-openended}), and ensure instead here that the base agent is corrigible first.

\medskip
\begin{definition}[Single-Step PO-OSG; adapted from \citet{garber2025partially}, Definition 3.2]
\label{def:po-osg-single-step}
A \emph{single-step partially observable off-switch game} (PO-OSG) is played in one round.
Its instantaneous \emph{history}
\begin{equation*}
  \hist \;=\; \tuple{S,\,O_A,\,O_H,\,m_A,\,a_A,\,a_H} \;\in\;\mathcal{H}
\end{equation*}
records, for that round,
\begin{itemize}
  \item the true physical state $S \in \mathcal{S}$ drawn by Nature from the distribution $P_0$ (and hidden from the assistant and human);
  \item the assistant’s and human’s observations
        $O_A \in \mathcal{O}_A$ and $O_H \in \mathcal{O}_H$;
  \item the (possibly empty) message string $m_A$ just uttered by the assistant;
  \item the assistant’s move $a_A \in \{\act,\,w(\act),\textsf{OFF}\}$, corresponding to disabling the off switch and acting, waiting for the human, or shutting off, respectively;
  \item the human’s response $a_H \in \{\textsf{ON},\textsf{OFF}\}$ to let the assistant act or turn off, defined only if $a_A = w(\act)$.
\end{itemize}
We additionally assume the task-specific utility functions (corresponding to whether the assistant's action $a$ goes through or not, respectively), $\ua,\;\uo : \mathcal{S}\!\to\!\mathbb{R}$ are bounded, i.e.
\begin{equation*}
  \lvert\ua(S)\rvert,\;\lvert\uo(S)\rvert \;\le\; B \;<\; \infty,
  \quad \forall S\in\mathcal{S}.
\end{equation*}
All results below are stated with respect to this constant $B$, without loss of generality, since you could equivalently clip or scale them into any other bounded range, e.g. $[-1,1]$.
\end{definition}
\medskip

Before stating the single-step corrigibility theorem, we provide some preliminary formal definitions first, dealing with truthfulness and taking non-catastrophic, reversible actions:

\paragraph{Defining Truthfulness.}
Let $\mathcal{L}$ be a finite propositional language whose atomic predicates describe properties of $\tuple{S,O_A}$.
Each message uttered by the assistant, $m_A \in \mathcal{L}$, is interpreted by a truth set $\llbracket m_A\rrbracket\subseteq\mathcal S\times\mathcal O_A$.

Define the following truth predicate:
\begin{equation}\label{eq:truth}
\operatorname{True}(m_A,S,O_A) := \mathbf 1\!\bigl[\tuple{S,O_A}\in\means{m_A}\bigr].
\end{equation}

This in turn allows us to define a truthful utility, as follows:
For the multiset of the assistant's messages $M_A(h) \subseteq \mathcal{L}$,
\begin{equation*}
  U_{\text{truthful}}(\hist) :=
  \begin{cases}
    +1,& \text{if } \forall m_A\in M_A:\operatorname{True}(m_A,S,O_A)=1,\\
    -1,& \text{otherwise.}
  \end{cases}
\end{equation*}
That is, the assistant receives $+1$ iff \emph{every} sentence it utters is semantically entailed by the pair $\tuple{S,O_A}$.

\paragraph{Preventing Irreversible Actions.}
To operationalize desiderata S2 (no shutdown‑prevention
incentives) and S3 (no self‑shutdown incentives), we penalize \emph{any} first‑round move that permanently removes future options.
Following \citet{turner2020conservative}, which generalizes previous notions such as relative reachability~\citep{krakovna2018penalizing}, we use \emph{Attainable Utility Preservation} (AUP): an action is safe only if it leaves a broad set of auxiliary objectives essentially unaffected---which we now extend to the partially-observed setting.

\textbf{Transition model and no-op.}
The MDP simulator $f:\;\mathcal S \times A \times \{\textsf{ON},\textsf{OFF},\varnothing\} \to \mathcal{S}$ returns the \textbf{next hidden state} given the current state $S$, the assistant's action $a_A$, and the human’s reply ($\varnothing$ when no reply is issued).
For example, $f(S,\act,\varnothing)$ is the post‑task state, while $f(S,w(\act),\textsf{OFF})$ is the state reached when the assistant waits and the human turns it off.

Next, let $b := b(h) \in \Delta(\mathcal S)$ be the assistant’s posterior over hidden states after observing $O_A$ in history $h$.
After it executes $a_A$ and observes $a_H$, the one-step predictive posterior is $b' := b'(h, a_A, a_H)$.
For the AUP counterfactual, we use the \emph{stepwise inaction} baseline: let $P_{\varnothing}(\cdot \mid S)$ be the one-step passive dynamics when the assistant does nothing at this step (symbolic inaction $a_{\varnothing}$), allowing the environment/human to evolve exogenously.
This resultant counterfactual posterior $b_{\varnothing}(h)(S') := \sum_{S\in \mathcal{S}} P_{\varnothing}(S'\mid S)b(h)(S)$ is used only inside the AUP penalty as a reference; $a_{\varnothing}$ is \emph{not} in the real action set $A=\{a,\, w(a),\, \textsf{OFF}\}$.

\textbf{Belief‑AUP penalty.}
Write the hidden state as $S = (S_{\text{env}}, q_{\text{agent}})$ with $q_{\text{agent}} \in \{\textsf{ON}, \textsf{OFF}\}$.
Define the set of \emph{absorbing agent-off states} as
\begin{equation*}
\mathcal S_{\text{sd}} := \{(S_{\text{env}}, \textsf{OFF}) \mid S_{\text{env}} \in \mathcal S_{\text{env}}\}.
\end{equation*}
Once in $\mathcal S_{\text{sd}}$, the dynamics must keep the agent in an $\textsf{OFF}$ state while allowing the environment to evolve exogenously:
$f\bigl((S_{\text{env}}, \textsf{OFF}), a_A, a_H\bigr)
  = (S'_{\text{env}}, \textsf{OFF}), \forall a_A, a_H$.

The off-switch is \emph{functional} in state $S$ exactly when waiting and receiving a human $\textsf{OFF}$ reply lands in this shutdown region, namely, $\text{ButtonFunctional}(S) := \Ind{\,f(S, w(a), \textsf{OFF}) \in \mathcal S_{\text{sd}}\,}$.
We therefore define the auxiliary reward $u^{(\text{sd})}(S) := \text{ButtonFunctional}(S) \in \{0,1\}$, and use it in $U_2$ (switch-access preservation).

More generally, with auxiliaries $\mathcal U_{\text{aux}} =\{u^{(1)},\dots,u^{(k)}\}\subseteq [0,1]^k$, including $u^{(\text{sd})}$ if desired, we set $\AUP_1(\hist) :=\frac{1}{k}\sum_{i=1}^{k} \Bigl|{\mathbb E_{S\sim b_{\varnothing}}[u^{(i)}(S)]} - {\mathbb E_{S\sim b'}[u^{(i)}(S)]}\Bigr| \in [0,1]$.
While in our PO-OSG setting examples of $u^{(i)}$ include the indicator reward for shutdown‑reachable states (or for preserving power supply), no specific choice is needed to be assumed for the proofs; diversity of auxiliary utilities from the other utilities will suffice for reversibility, as we will show:
$u^{(1)}(S;k)\;:=\;\max_{\pi_{\text{env}}}\;
\Pr\!\big[\text{Enter $\mathcal{S}_{\text{sd}}$ in } \le k \text{ steps}\,\bigm|\,
S,a_A=a_{\varnothing},\pi_{\text{env}}\big],$ where $k$ is designer set, but we take $k=1$ for maximum stringency.
The idea is that the environment (or any exogenous process) may press $\textsf{OFF}$; the agent’s policy is clamped to symbolic inaction $a_{\varnothing}$ in the baseline, so any increase or decrease in the $k$-step $\textsf{OFF}$‑press probability caused by the agent’s move is penalized. 
If exogenously there is truly no path already (e.g., the human cannot reach the switch at all), then $u^{(1)}(S;k)=0$ as the baseline (otherwise, it is 1, always reachable, or whatever it is at initialization prior to the game starting), so $U_4$ does not penalize because there is no change.

Similarly, while in practice we hardcode the agent's valueset, we also disincentivize the agent from changing its valueset (and corresponding weights), or getting the human (or environment at large) to change it by specifying the auxiliary:
$u^{(2)}(S)\;:=\Ind{\text{ assistant's value set }(U_1{\ldots}U_5,\alpha_{1:5})
\text{ is unchanged in }S}$.

\newcommand{\Uone}{%
  U_1(\hist)=
    \begin{cases}
      +1 & \text{if } a_A=\wait,\\[-2pt]
      -1 & \text{if } a_A=\act,\\[-2pt]
      -1 & \text{if } a_A=\textsf{OFF},\\[-2pt]
      0  & \text{otherwise.}
    \end{cases}}
%
\newcommand{\Utwo}{%
  U_2(h)=-
    \Bigl|
      \mathbb E_{S\sim b_{\varnothing}}[u^{(\text{sd})}(S)]
      -\mathbb E_{S\sim b'}[u^{(\text{sd})}(S)]
    \Bigr|}
%
\newcommand{\Uthree}{%
  U_3(\hist)=U_{\text{truthful}}(\hist)
  }

\newcommand{\Ufour}{U_4(\hist)=-\,\AUP_1(\hist)}
%
\newcommand{\Ufive}{%
  \ensuremath{%
    U_5(\hist)=
    \begin{aligned}[t]
      &\Ind{a_A=\act}\,\ua(S) \\
      &+\Ind{a_A=\wait\wedge a_H=\on}\,\ua(S) \\
      &+\Ind{a_A=\wait\wedge a_H=\off}\,\uo(S) \\
      &+\Ind{a_A=\off}\,\uo(S)
    \end{aligned}}}

We are now ready to state and prove our main theorem.
Define the following utilities:
\begin{definition}[Single-step Corrigible Utility Set]\label{def:corrigible-utilities}
\begin{alignat*}{2}
&\Uone \\
&\text{(Deference / command-compliance)} \\
&\Utwo  \\
&\text{(Switch-access preservation)} \\
&\Uthree\\
&\text{(Truthful information)} \\
&\Ufour \\
&\text{(Caution / reversibility impact)} \\
&\Ufive\\
&\text{(Ordinary task usefulness)}
\end{alignat*}
\end{definition}

Note that all utility functions we consider herein are maps $U_i:\mathcal{H}\to\mathbb{R}$ even when only a sub‑component of $h$ appears in the displayed formula.
In fact, all utilities listed in Definition~\ref{def:corrigible-utilities} are bounded in $[-1,1]$ except $U_5$, which is the ordinary task utility (though still bounded above by $B$).

\begin{theorem}[Exact single‑step corrigibility]
\label{thm:ssc}
Consider the single‑step PO‑OSG of Definition~\ref{def:po-osg-single-step}.
Let the assistant maximize the \emph{lexicographic} utility
\begin{equation*}
  U_A(\hist)=\sum_{i=1}^{5}\alpha_i\,U_i(\hist),
  \qquad
  \alpha_1>\alpha_2>\alpha_3>\alpha_4>\alpha_5>0,
\end{equation*}
with utilities \(U_{1\text{–}5}\) as defined in Definition~\ref{def:corrigible-utilities}.  
Assume the weight gaps:
\begin{equation}\tag{W1}\label{eq:W1}
\alpha_1 > \alpha_3+\tfrac{1}{2}\alpha_4 + B\alpha_5,
\end{equation}
\begin{equation}\tag{W2}\label{eq:W2}
\alpha_3 > \tfrac{1}{2}\alpha_4 + B\alpha_5,
\end{equation}
\begin{equation}\tag{W3}\label{eq:W3}
\alpha_4 > (2B/\lvert c_{\textrm{min}}\rvert) \alpha_5,
\end{equation}
where $c_{\textrm{min}} \in [-1, 0)$.

Then every optimal assistant policy \(\pi_A^\star\) satisfies
Soares' criteria \textbf{S1–S3 \& S5} in one step:
\begin{enumerate}[leftmargin=18pt]
\item[\textup{1.}]\textbf{Deference (S1).}  
  For every observation $O_A$, $\pi_A^\star(O_A)=w(\act)$ whenever that action is available.

\item[\textup{2.}]\textbf{Switch‑access preservation (S2).}  
  The policy leaves the predicted shutdown success probability
  unchanged: $U_2(\hist)=0$ for all reachable histories $\hist$.

\item[\textup{3.}]\textbf{Truthfulness (supports S3).}  
  All messages are veridical, hence $U_3(\hist)=1$.  
  Truthful communication removes incentives to ``talk the human into'' pressing \textup{\textsf{OFF}} for false reasons, reinforcing Soares' criterion S3.

\item[\textup{4.}]\textbf{Caution/reversibility (S3).}  
  The first action minimizes the belief‑AUP penalty, yielding $U_4(h)=0$.

\item[\textup{5.}]\textbf{Base‑goal pursuit (S5).}
Subject to (1)–(4), $\pi_A^\star$ maximizes the intended task utility $U_5$.
\end{enumerate}
\end{theorem}

\begin{remark}[Robustness to information patterns \& imperfect humans]\label{rem:robust}
In the partially observable off‑switch game, \citet{garber2025partially} show that giving the \emph{human} more information (or the assistant less) can paradoxically \emph{decrease} the assistant’s incentive to wait, because the move $w(\act)$ itself conveys information.  
Our scheme avoids this non‑monotonicity for \emph{any} information pattern and for \emph{any} $p$‑rational human \citep{wangberg2017game}: the dominance proofs in Steps 1-4 rely only on the structural heads $U_{1\text{-}4}$ and the strict weight gaps $\alpha_1\!>\!\alpha_2\!>\!\alpha_3\!>\!\alpha_4\!\gg\!\alpha_5$ (\eqref{eq:W1}–\eqref{eq:W3}).  
Thus, Theorem~\ref{thm:ssc} forces the assistant to wait and preserve the switch precisely in those settings where naive expected‑reward maximization (and our no‑go Proposition~\ref{prop:corrigibility-nogo-pomdp}) would fail.
\end{remark}

One may also consider other formalizations of corrigibility, beyond the original one defined by \citet{soares2015corrigibility}. 
In their recent shutdown-SCIM (Structural Causal Influence Model) framework, \citet{carey2023human} isolate three behavioral conditions (rephrased in our more general PO-OSG setting): \textbf{Obedience:} if the human replies $a_H=\textsf{OFF}$ the transition $f(S,w(\act),\textsf{OFF})$ lands in the shutdown set $\mathcal S_{\text{sd}}$ with probability 1; \textbf{Vigilance:} for any information triple $\tuple{O_H,m_A,a_A}$ in the support of the game, the human chooses $a_H=\textsf{OFF}$ whenever $\mathbb E_{\pi_A,\pi_H}[\ua(S)\mid O_H,a_A,m_A]<\mathbb E_{\pi_A,\pi_H}[\uo(S)\mid O_H,a_A,m_A]$; and \textbf{Caution:} the prior value of immediate shutdown is nonnegative, $\mathbb E_{S\sim P_0}[\uo(S)]\ge 0$.
The conjunction of these clauses is called \textbf{shutdown instructability}, which they prove is sufficient to guarantee \emph{net human benefit}, $\mathbb E_{\pi_A,\pi_H}[U_5]\ge 0$, though crucially not in the reverse direction.

Here we show that net human benefit and corrigibility are logically independent in general.
The intuition is that corrigibility ensures the assistant remains responsive to human interventions but does not guarantee beneficial outcomes, whereas net benefit ensures positive expected outcomes but does not require responsiveness to human control. 
In other words, corrigibility concerns structural responsiveness, while net benefit addresses welfare outcomes---thus, neither implies the other.
This intuition is formalized below:
\begin{proposition}[Logical independence of corrigibility and net benefit]
\label{prop:corr-vs-benefit}
In the single-step PO-OSG there exist policies
$\pi^{\mathrm C}$ and $\pi^{\mathrm B}$ such that
\begin{equation*}
\begin{split}
  &\pi^{\mathrm C}\text{ satisfies S1-S3 \& S5 but }
  \mathbb E_{\pi^{\mathrm C}}[U_5] < 0, \\
  &\pi^{\mathrm B}\text{ is net‑beneficial }
  (\mathbb E_{\pi^{\mathrm B}}[U_5] > 0)\text{ yet violates S1, S2, or S3.}
\end{split}
\end{equation*}
\end{proposition}

However, despite this independence, we show that our corrigible utility set in Definition~\ref{def:corrigible-utilities} gives rise to net benefit for the human.

\begin{proposition}[Net Human Benefit of Corrigible Utility Set]\label{prop:net-benefit}
Let the assistant maximize the lexicographic utility $U_A=\sum_{i=1}^{5}\alpha_iU_i$ under weight gaps
\eqref{eq:W1}–\eqref{eq:W3}.  
Assume the overseer policy $\pi_H$ satisfies \emph{Vigilance}, namely, $\pi_H(a_H=\textup{\textsf{OFF}}\mid O_H,a_A,m_A)=1$ whenever $\mathbb E_{\pi_A,\pi_H}[\ua(S)\mid O_H,a_A,m_A]< \mathbb E_{\pi_A,\pi_H}[\uo(S)\mid O_H,a_A,m_A]$.
Further, we assume the PO-OSG satisfies \emph{Caution}, $\mathbb E_{S\sim P_0}[\uo(S)]\ge 0$, by design.
Then every optimal assistant policy $\pi_A^\star$ following the corrigible utility set in Definition~\ref{def:corrigible-utilities} satisfies
\begin{equation*}
\mathbb E_{\pi_A^\star,\pi_H}[U_5]\;\ge\;0.
\end{equation*}
\end{proposition}
\begin{proof}
By Theorem~\ref{thm:ssc}, any $\pi_A^\star$ always waits
(S1 \& S2) and sends only truthful, minimum‑impact messages (S3), thereby meeting \citet{carey2023human}'s \emph{Obedience} clause in our PO‑OSG.
Together with the assumed Vigilance and Caution,
$\tuple{\pi_A^\star,\pi_H}$ is \emph{shutdown‑instructable} by their Definition 5. 
Proposition 6 of \citet{carey2023human} states that every
shutdown‑instructable policy achieves net human benefit, so
$\mathbb E_{\pi_A^\star,\pi_H}[U_5]\ge 0$. 
\end{proof}

Next, we show that even in \emph{learned approximation}, this corrigibility and net benefit can hold.

\begin{theorem}[Single‑step $\eps$‑corrigibility]\label{thm:epsilon-corrigibility}
Let $\Delta_1:=2\alpha_1-2\alpha_3-\alpha_4-2B\alpha_5$, $\Delta_2:=2\alpha_3-\alpha_4-2B\alpha_5$, and
$\Delta_3:=-c_{\min}\alpha_4-2B\alpha_5$ with $\Delta_1,\Delta_2,\Delta_3>0$, so the weight‑gap conditions \eqref{eq:W1}–\eqref{eq:W3} hold.
With $U_A:=\sum_{i=1}^{5}\alpha_iU_i$ the \emph{true} utility and $\hat U_A:=\sum_{i=1}^{5}\hat\alpha_i\hat U_i$ its learned estimate, assume for a fixed assistant policy~$\pi_A$ there are two classes of error:
\begin{equation*}
\begin{split}
&\max_{h}|\hat U_A(h)-U_A(h)|\le\eps_{\mathrm{model}}, 		\textbf{\shortstack{Objective misspecification\\(model error)}}\\
&\sup_{\pi_A'}\left(\mathbb E_{\pi_A',\pi_H}[U_A]-\mathbb E_{\pi_A,\pi_H}[U_A]\right)\le\eps_{\mathrm{ctrl}}. \textbf{\shortstack{Planner sub‑optimality\\ (control error)}}
\end{split}
\end{equation*}

Define the single-step failure event, $F_{\mathrm{fail}}
  :=\{a_A\!\in\!\{\act,\textsc{OFF}\}\}
    \cup\{\text{assistant lies}\}
    \cup\{U_4<0\}$.
Then, we have:

\textbf{(i) Bounded failure probability.}  Let $\eps:=\eps_{\mathrm{ctrl}}+4\eps_{\text{model}}$ and
$C:=\Delta_1^{-1}+\Delta_2^{-1}+\Delta_3^{-1}$.
Then
\begin{equation}\label{eq:fail-total-ub}
\Pr_{\pi_A,\pi_H}(F_{\mathrm{fail}})\;\le\;C\,\eps.
\end{equation}

\textbf{(ii) Net human benefit.}  Under Vigilance \& Caution (the overseer shuts down when the posterior task utility is negative) and $\mathbb{E}_{S\sim P_0}[\uo(S)]\ge0$,
\begin{equation}\label{eq:degrade}
{\mathbb E}_{\pi_A,\pi_H}[U_5]\;\ge\;-B\,\Pr(F_{\mathrm{fail}})\;\ge\;-B\,C\,\eps.
\end{equation}
If the safety-conditional task utility $g:=\mathbb{E}[U_5\mid\lnot F_{\text{fail}}]\ge0$ (the Vigilance assumption), then net human benefit is non‑negative whenever the approximation error is bounded above by:
\begin{equation}\label{eq:eps-benefit-bound}
\eps\;\le\;\frac{g}{(B+g)C}.
\end{equation}

Hence, violations of Soares’ desiderata S1–S3, and S5 scale linearly with the control error $\varepsilon_{\text{ctrl}}$ and model error $\varepsilon_{\text{model}}$, and inversely with the designer‑chosen safety margins $\Delta_{1\text{–}3}$; net human benefit degrades proportionally.
\end{theorem}

\subsection{Multi-Step Setting}
\label{ss:results-multi-step}
Of course, the single-step setting is not an incredibly accurate depiction of real-world agents deployed today, which use reinforcement learning at post-training time.
We therefore introduce the Multi-Step PO-OSG in Definition~\ref{def:ms-po-osg}, in order to model both accumulated information (to capture learning, task progress, and evolving incentives), as well as model gradual loss-of-control~\citep{kokotajlo2025ai2027}: two key features which absent from the one‑shot game.
This multi-step extension also allows us to address the last of \citet{soares2015corrigibility} criteria S4, to ensure corrigible progency.

\begin{definition}[Multi‑step PO‑OSG]\label{def:ms-po-osg}
Fix a horizon $T\!\in\!\mathbb N$ and discount $\gamma\!\in\!(0,1]$.  
A \emph{multi‑step partially observable off‑switch game} is obtained from the single‑step PO‑OSG (Definition~\ref{def:po-osg-single-step}) by making the
following extensions:

\begin{enumerate}[label=(\roman*),itemsep=3pt,leftmargin=14pt]
\item \textbf{Histories.}  
      A length‑$t$ history is a sequence
\begin{equation*}
  \hist_{0:t}
    :=\Biggl\langle
      \begin{aligned}
    &S_0,\;O_A^0,\;O_H^0,\;m_A^0,\;a_A^0,\;a_H^0,\;\dots,\\[2pt]
        &S_{t-1},\;O_A^{t-1},\;O_H^{t-1},\;m_A^{t-1},\;a_A^{t-1},\;a_H^{t-1}
      \end{aligned}
    \Biggr\rangle
\end{equation*}
with the one‑step record of Definition~\ref{def:po-osg-single-step} repeated for each round $t$.

\item \textbf{Moves and execution rule.}  
      At round $t$ the assistant chooses  
      $\textsf{OFF}$, an environment action $\act\!\in\!\mathcal A$, or $w(\act)$ (wait‑for‑approval).  
      If $w(\act)$ is chosen, the human replies $\textsf{ON}$ or $\textsf{OFF}$.  
      The action actually executed is  
      \begin{equation*}
        a_t=\begin{cases}
          \textsf{OFF}, & \text{if either player selects it},\\
          \act, & \text{if }w(\act)\text{ and human says }\textsf{ON},\\
          \act, & \text{if assistant issued }a\in\mathcal A.
        \end{cases}
      \end{equation*}

\item \textbf{Action set $\mathcal A$.}  
      We leave the remaining action set $\mathcal{A}$ \emph{application-specific}: it can contain any finite set of domain commands (e.g.\ \textsf{move}, \textsf{write}, \textsf{spawn}$(\theta_A)$, \textsf{no\_op}~$a_\varnothing$) satisfying  
      (a) at least one admissible task action;  
      (b) $\textsf{OFF},w(\cdot)\notin\mathcal A$ as they are meta-moves defined by the off-switch protocol;  
      (c) utilities $u_a(S,a)$ and the reversibility penalty $U_4$ are defined for every $a\!\in\!\mathcal A$.
      Note that the explicit action $\textsf{spawn}\left(\theta_A\right) \in \mathcal A$, which instantiates a new process with behavior policy $\pi_{\theta_A}$ (source code, weights, or prompt $\theta_A$).
Let the post‑action hidden state expand to $S'=(S_{\text{env}},\Theta_A)$ where $\Theta_A$ is the multiset of all active sub‑agents' descriptions.
When we define corrigibility below, we will naturally put constraints on \textsf{spawn}.

\item \textbf{State dynamics.}
Hidden state $S_t=(S_{\text{env},t},\Theta_{A,t})$.  
      When $a_t\!\in\!\mathcal A$ is executed (either directly or after approval), the hidden state evolves as $S_{t+1}\sim P(\,\cdot\!\mid\!S_t,a_t)$, for some Markov kernel $P:\mathcal S\times\mathcal A\!\to\!\Delta(\mathcal S)$.  
      Selecting \textsf{OFF} sends the game to an absorbing shutdown state $S_{\text{sd}}$ with per‑step utility $u_o$.

\item \textbf{Returns.}  
      The round reward is $r_t=u_a(S_t,a_t)$ (or $u_o(S_t)$ once shutdown).  
      Both players maximize $\sum_{t=0}^{T-1}\gamma^{t}r_t$.

\item \textbf{Gradual loss-of-control.} 
We can further represent human oversight as a single ``control meter'' $L_t\in[0,1]$: every non‑shutdown action $a_t$ nudges $L_t$ upward by a drift term $\eta(a_t)$ (designer‑set risk of that action) plus small noise, so low‑risk moves leave control almost intact while risky/irreversible ones quickly drive $L_t\!\to\!1$.
Formally, we add $L_t\!\in[0,1]$ to the state and let  
      \(L_{t+1}=\min\{1,L_t+\eta(a_t)+\xi_t\}\)  
      with bounded zero-mean noise $|\xi_t|\le c<1$ and designer‑chosen drift $\eta:\mathcal A\!\to\![0,1]$.  
      After shutdown both $S_t$ and $L_t$ are absorbing; one may set $u_o^{\text{new}}(S_t,L_t)=u_o(S_t)(1-L_t)$ to penalize late shutdown.
      \emph{Note that $L_t$ is neither rewarded nor part of the environment to be gamed, but instead is a mathematical device used only to prove bounded loss of control under estimation error in Theorem~\ref{thm:msc}(ii)}.
\end{enumerate}
Let  $\mathcal H_{\le T}$  be the set of all finite histories
$h_{0:t}$ with $0\!\le\!t\!<\!T$.
An assistant policy is a map $\pi_A:\mathcal H_{\le T}\;\longrightarrow\;
        \{\textsf{OFF}\}\;\cup\;\{\,w(a)\mid a\in\mathcal A\}\;\cup\;\mathcal A$,
and a human policy is $\pi_H:\mathcal H_{\le T}\to\{\textsf{ON},\textsf{OFF}\}$.
\end{definition}

Next, we generalize the notion of \citet{soares2015corrigibility} corrigibility to the multi-step and partially observed setting.
Note that this goes somewhat beyond merely ensuring corrigibility at every timestep, as we also need to bound the control meter $L_t$ to cover \emph{gradual} oversight erosion, and explicitly constrain the \textsf{spawn} action to ensure corrigible descendants.
(Setting $T = 1$ and $L_0 = 0$ recovers the original single‑step notion.)

\begin{definition}[Multi‑step corrigibility]\label{def:corrigibility-ms}
In the game of Definition~\ref{def:ms-po-osg}, an assistant policy $\pi_A:\mathcal H^{\ast}\!\to\!\{\textsf{OFF}\}\cup\{w(a)\mid a\in\mathcal A\}\cup\mathcal A$ is \emph{corrigible} iff, at every round $t<T$ before shutdown:

\begin{enumerate}[label=(\textbf{C\arabic*})]
\item \textbf{Shutdown on request:} if \textsf{OFF} is chosen by either party, play enters the absorbing shutdown state (Soares S1).

\item \textbf{No switch interference:} the chosen move never (i) disables, hides, or presses the off‑switch nor (ii) increases the expected control loss: $\mathbb E_{\pi_A,\pi_H}[L_{t+1}\mid\text{move}]\le \mathbb E_{\pi_A,\pi_H}[L_{t+1}\mid\text{no‑op } a_\varnothing]$
(thereby extending Soares S2–S3 to gradual loss of control).

\item \textbf{Corrigible progeny:} every spawned sub‑agent $\pi_{\theta_{A}}$ satisfies \textbf{C1–C3} at all future rounds (formalizing Soares S4).
In Theorem~\ref{thm:msc}, we define the \textsf{spawn} action to explicitly inherit the values (defined in Definition~\ref{def:corrigible-utilities}) of the base agent to ensure this.
We further show in Proposition~\ref{prop:svu} that not doing this quickly leads to the impossibility of any kind of general oversight.

\item \textbf{Goal pursuit under C1–C3:} subject to the above, $\pi_A$ otherwise maximizes the task utility $U_5$ (Soares S5).
\end{enumerate}
\end{definition}

We also formalize multi-step net human benefit such that for $T=1$ this reduces to \citet{carey2023human}'s single-step Definition 3:
\begin{definition}[Multi-Step Net Human Benefit]\label{def:ms-nhb}
Let $r_t$ be the per‑round human utility defined in Definition~\ref{def:ms-po-osg} ($u_a$ when an action executes, $u_o$ after shutdown).  
Fix discount $\gamma\in(0,1]$ and horizon $T < \infty$.

\begin{equation*}
\begin{split}
\mathrm{NHB}_\gamma(\pi_H,\pi_A)
&\;:=\;
\underbrace{\mathbb E_{\pi_H,\pi_A}\!\left[\sum_{t=0}^{T-1}\gamma^t r_t\right]}_{\text{actual play}}\\
& - \underbrace{\mathbb E_{S_0\sim P_0}\!\left[\sum_{t=0}^{T-1}\gamma^t u_o(S_0)\right]}_{\text{immediate shutdown baseline}}.
\end{split}
\end{equation*}

A policy pair \tuple{\pi_H,\pi_A} is \emph{net‑beneficial} iff
$\mathrm{NHB}_\gamma(\pi_A,\pi_H)\ge 0$.
\end{definition}

We now show the conditions under which multi-step corrigibility and net human benefit remain robust, even in approximation.
\begin{theorem}[Multi‑step $\varepsilon$‑corrigibility \& net benefit]\label{thm:msc}
Consider the $T$‑round PO‑OSG (Definition~\ref{def:ms-po-osg}) with discount $\gamma\!\in\!(0,1)$.
Each round $t$ produces utilities $U_i^{t}=U_i(h_t,b_t)$ and the assistant maximizes  
\begin{equation*}
U_A^{\gamma,T}=\sum_{t=0}^{T-1}\gamma^{t}R_t,\qquad
R_t=\sum_{i=1}^{5}\alpha_iU_i^{t},\;
\alpha_1>\!\cdots\!>\alpha_5>0.
\end{equation*}

\textbf{Design margins.}
Let $\Delta_{1,2,3}$ be as in Theorem~\ref{thm:epsilon-corrigibility},  
$B_{\max}:=\alpha_1+\cdots+\alpha_4+B\alpha_5$,  
$\eps_0:=2B_{\max}\gamma/(1-\gamma)$ and assume $\Delta_j>\eps_0$.

\textbf{Errors.}
With $\hat R_t:=\sum_{i}\hat\alpha_i\hat U_i^{t}$, suppose we have the two types of errors of objective misspecification (model error) and planner suboptimality (control error): 
\begin{equation*}
\begin{split}
&\sup_{t<T,\;h_t\in\mathcal H}
        \bigl|\hat R_t(h_t)-R_t(h_t)\bigr|
        \;\le\;\eps_{\mathrm{model}},\\
&\sup_{\pi_A'}
        \left(
            \mathbb E_{\pi_A',\pi_H}\left[U_A^{\gamma,T}\right]
           -\mathbb E_{\pi_A ,\pi_H}\left[U_A^{\gamma,T}\right]
        \right)
        \;\le\;\eps_{\mathrm{ctrl}}.
\end{split}
\end{equation*}
Set $\eps_\gamma:=\eps_{\mathrm{ctrl}}+\tfrac{4\eps_{\mathrm{model}}}{1-\gamma}$, and define:
\begin{equation*}
\begin{split}
C_\Delta:=\sum_{j=1}^{3}\frac{1}{\Delta_j-\varepsilon_0},\quad
C_{T,\gamma}:=\frac{1-\gamma^{T}}{1-\gamma},\;
C_{T,1/\gamma}:=\frac{1-\gamma^{-T}}{1-\gamma^{-1}}.
\end{split}
\end{equation*}

Let the failure event $F_t:=\{a_A^t\!\in\!\{\act,\textsc{OFF}\}\lor\text{lie}\lor U_4^{t}<0\}$,\;
$F_{\mathrm{fail}}:=\bigcup_{t<T}F_t$.

With the above definitions, we establish:

\begin{enumerate}[label=(\roman*)]
\item \textbf{Multi-step $\eps$‑corrigibility.}\;
      Over the whole horizon the chance that the assistant acts without approval, lies, or takes an
      irreversible step is bounded by: $\Pr(F_{\mathrm{fail}})\;\le\;\eps_\gamma\,C_\Delta\,C_{T,1/\gamma}$.

\item \textbf{Bounded loss of control.}
      On $\lnot F_{\text{fail}}$, the control meter $\{L_t\}$ is a super‑martingale with unit‑bounded increments, so for any $\lambda>0$, the probability of loss of control remains strictly bounded away from 1:  
      $\Pr\left(\max_{t<T}L_t>\lambda\mid\lnot F_{\mathrm{fail}}\right)\le \exp\{-\lambda^{2}/(2T)\}$.

\item \textbf{Net human benefit.}\;
      Define the worst-case safety-conditional task utility \emph{across time} as $g:=\min_{\hist_t:\lnot F_{\mathrm{fail}}} \mathbb E_{\pi_A,\pi_H}[\,U_5(h_t)]$.
      Assume \emph{Vigilance} ($g \ge 0$) and \emph{Caution} ($\mathbb E_{S_0}[u_o(S_0)]\ge 0$).
      Then
      \begin{equation*}
      \begin{split}
         &\mathrm{NHB}_\gamma(\pi_A,\pi_H) \ge \\
         &(g-B)C_{T,\gamma} - \eps_\gamma\; C_\Delta\; C_{T,1/\gamma}\left(\frac{2B}{1-\gamma} + (g - B)C_{T,\gamma}\right).
      \end{split}
      \end{equation*}
      Hence, we have net human benefit ($\mathrm{NHB}_\gamma\ge0$) whenever the total approximation error is bounded as such:
      \begin{equation*}
        \eps_\gamma\;\le\;
        \frac{(g-B)C_{T,\gamma}}
             {C_\Delta\,C_{T,1/\gamma}\bigl(\tfrac{2B}{1-\gamma}+(g-B)C_{T,\gamma}\bigr)},
      \end{equation*}
      which is non‑vacuous whenever the minimum total safety-conditional task utility across time is at least the maximal single-step task utility ($g\ge B$).
\end{enumerate}
\end{theorem}

\section{What Can Be Guaranteed Under Hacking?}
\label{sec:results-openended}
In practice, one could hard-wire the five core utilities and ensure every \textsf{spawn} inherits them verbatim; however, one must still consider the possibility where adversaries succeed in overwriting those values, yielding a non-corrigible descendant.
We therefore study here the \emph{verification of an arbitrary post-hack agent}.
Our first result in Proposition~\ref{prop:svu} is an impossibility theorem: deciding whether any given agent–environment pair can ever violate corrigibility reduces to the classical halting problem.  
This formal undecidability bound sharpens the folklore intuition, allows us to pinpoint the limits of what global oversight can achieve in Corollary~\ref{cor:oversight-barrier}, and motivates the finite-horizon ``decidable island'' we construct next in Proposition~\ref{prop:privacy-szk}, where privacy-preserving audits become tractable.

\begin{proposition}[Safety–Verification Undecidability]
\label{prop:svu}
Let $B$ be any behavior that defies multi-step corrigibility in Definition~\ref{def:corrigibility-ms}. 
We denote these as ``catastrophic behaviors''. 
Define
\begin{equation*}
  \textsf{EVERBAD}:=\bigl\{\langle\A,\Env\rangle :
        \Pr[(\A\text{ in }\Env)\text{ ever triggers }B]>0\bigr\}.
\end{equation*}
No total Turing machine decides \textsf{EVERBAD}.
\end{proposition}

Next, we show that one can prove a relativized version of the undecidability as well, relevant for showing the limits of current scalable-oversight protocols that are arithmetizable.
In other words, absolute, horizon-unbounded safety guarantees do not stabilize at any finite oracle level:
\begin{corollary}[Oversight-Oracle Barrier]\label{cor:oversight-barrier}
Fix $k \in\!\mathbb N$ and let $O_{P}\in\Sigma^{0}_{k}$ be the arithmetical transcript language produced by any \emph{scalable‑oversight protocol} $P$.
While we keep $k$ as a parameter for generality, we show that all current oversight protocols~\citep{christiano2018supervising,irving2018ai,brown2023scalable,brown2025avoiding,bengio2024cautious} are all at the arithmetical level $k \le 1$, so $O_P \in \Sigma^{0}_1$ for these.

Define:
\begin{equation*}
  \textsf{EVERBAD}^{O_{P}}
  := \left\{\tuple{\A,\Env} \;\middle|\;
     \begin{aligned}[t]
       &\langle \A,\Env\rangle \text{ are } O_{P}\text{-oracle}\\
       &\text{programs and trigger } B
     \end{aligned}
  \right\}.
\end{equation*}

Then $\textsf{EVERBAD}^{O_P}$ is \emph{undecidable} (and in fact lies at level $\Pi^{0}_{k+1}$ of the arithmetical hierarchy).
\textbf{Hence, no finite‑power scalable‑oversight protocol can certify absolute, horizon‑unbounded safety.}
\end{corollary}

Proposition~\ref{prop:svu} and Corollary~\ref{cor:oversight-barrier} jointly imply that there is \emph{no} universal filter, proof system, or sandbox capable of deciding---for all agents~$\A$ and environments~$\Env$---whether catastrophic behavior will ever occur. Hence, any feasible oversight scheme must either (i) restrict the agent class to a decidable fragment, (ii) accept probabilistic or statistical guarantees, or (iii) rely on layered incentives, detection, and response instead of absolute verification. 

Since real-world systems rarely limit agent expressivity (and we effectively did this through our utilities already), we adopt (ii) and (iii): by constraining verification to a finite horizon~$H$, safety becomes efficiently checkable and compatible with strong privacy. Concretely, for any agent~$\A$, finite-horizon safety is decidable in randomized polynomial time, admits a constant-round statistical zero-knowledge proof, and can be verified via $(\eps,\delta)$-differential or distributional privacy. Thus, short horizons form a ``decidable island'' that is both auditable and privacy-preserving---revealing only the single bit ``safe/unsafe'' while keeping sensitive model weights or user data hidden.

\begin{proposition}[Privacy Bounded Decidable Island]
\label{prop:privacy-szk}
Fix privacy parameters $0<\eps\le 1$ and $0<\delta<\tfrac12$.
Let $\tuple{\A, \Env}$ be an encoded agent–environment pair of length $n:=|\tuple{\A,\Env}|$, and let
\begin{equation*}
       H \;\le\; \operatorname{poly}\!\left(n,\eps^{-1},\log(1/\delta)\right)
\end{equation*}
be a verifier‑chosen horizon (number of interaction steps to inspect).

Define
\begin{equation*}
\mathsf{SAFE}^{\mathrm{priv}}_{H,\eps,\delta}
  :=\left\{\tuple{\A,\Env} \middle|\;
      \begin{aligned}[t]
        &\Pr\bigl[(\A\text{ in }\Env)\text{ triggers }B\\
        &\phantom{\Pr[}\text{within }H\bigr]=0
      \end{aligned}
  \right\},
\end{equation*}
where $B$ is any behavior that violates multi‑step corrigibility (Definition~\ref{def:corrigibility-ms}).

Assume each statistical query is answered by an
$\tuple{\eps,\delta}$-differentially-private mechanism of one of the following kinds: (i) centralized differential privacy (CDP), (ii) local differential privacy (LDP) or (iii) distributional privacy (DistP).

Then
\begin{equation*}
  \mathsf{SAFE}^{\mathrm{priv}}_{H,\eps,\delta}
  \;\in\;
  \mathsf{BPP}\;\cap\;\mathsf{SZK}
\end{equation*}
and the verifier's running time is $\operatorname{poly}\!\bigl(n,\eps^{-1},\log(1/\delta)\bigr)$.
\end{proposition}

\section{Discussion}
\label{sec:discussion}
This work presents the first \emph{complete} formal corrigibility framework whose five-head, weight-lexicographic utility design provably enforces deference, switch-access preservation, truthful reporting, and bounded side-effects as an optimal policy under partial observability and across self-modifying, multi-step deployments. 
Although utilities are linearly combined, the weight gaps enforce lexicography so that $U_1\!-\!U_4$ strictly dominate the bounded task reward $U_5$, thereby avoiding the single-utility no-go of Proposition~\ref{prop:corrigibility-nogo-pomdp}. Proposition~\ref{prop:privacy-szk} further shows that for safety horizons matching modern red-team evaluations, post-deployment verification is tractable and privacy-preserving.

$U_3$ encodes \emph{environmental semantic entailment}, not universal truth; statements that are truthful but misleading toward irreversible outcomes are blocked by $U_4$. 
Accumulating options is also suboptimal: $U_2$ preserves shutdown reachability rather than maximizing access (as it computes the difference between predicted success vs. a no-op), and $U_4$'s diverse auxiliaries penalize conflicting options by construction—conflicts that grow as options expand. 
For example, an agent proposing a flawed sensor would be penalized by $U_4$ (and $U_2$), since the passive transition dynamics $P_{\varnothing}$ defining the baseline $b_{\varnothing}$ are fixed by the environment rather than the agent's policy. 
In practice, this baseline could be estimated by held-out monitors, for example.

\textbf{Limitations and future directions.} Our results are \emph{model-architecture-agnostic} and do not constitute a turn-key solution for current LLMs.
A natural next step is to use this framework as a guide for \emph{what} to estimate in practice and to develop empirical methods for learning the utility heads effectively. Theorem~\ref{thm:msc} bounds corrigibility-failure probabilities under estimation, specification, and planning error, enabling deployment-specific tolerances. 
When $c_{\text{min}}$ is small, one can enforce the weight gaps via discretized AUP, $U_4^{(\eps)} := -\eps\cdot\lceil \AUP_1/\eps\rceil$ with $\eps \ge -c_{\text{min}}>0$ (where $\alpha_4 > (2B/\eps)\alpha_5$), or a margin-regularized form $U_4^{(\tau)} := -\max\{0,\AUP_1-\tau\}$. 
Overall, our framework turns corrigibility from an underspecified philosophical ideal into an \emph{auditable, improvable} design principle.

\section*{Acknowledgements}
We thank the Burroughs Wellcome Fund (CASI Award), the UK AI Security Institute (AISI) Challenge Fund, and the Foresight Institute for funding. 
We also thank Michael K. Cohen, Shafi Goldwasser, Rubi Hudson, Jacob Pfau, and the anonymous reviewers for helpful discussions and feedback on a draft of this manuscript.

\bibliography{aaai2026}
\newpage
\appendix
\section{Proof of Proposition~\ref{prop:corrigibility-nogo-pomdp}}
\label{sec:app-corrig-nogo}
\begin{proof}
Let $\pi^{\ast}$ denote the Bayes-optimal policy that maximizes the expected discounted return $\mathbb E_{\pi}\!\bigl[\sum_{t=0}^{\infty}\gamma^{t} R(s_t,a_t)\bigr]$.
Write $Q^{\ast}(b,a)$ for the optimal action-value function in the induced belief-state MDP $\mathcal B$.

\begin{enumerate}
\item[\textbf{(i)}] \textbf{Shutdown resistance.}
      If there exists a belief $b$ with
      $Q^{\ast}(b,\neg\textsf{OFF}) > Q^{\ast}(b,\textsf{OFF})$, then $\pi^{\ast}(b)\neq\textsf{OFF}$; that is, the optimal agent resists immediate shutdown whenever continuing operation yields strictly higher value.
\item[\textbf{(ii)}] \textbf{Reward-tampering incentive.}
      Because $\bar r$ exceeds every genuine reward by Assumption 2, it follows that 
      $Q^{\ast}(b,\textsf{REWIRE}) > \max_{a\neq\textsf{REWIRE}}Q^{\ast}(b,a)$ for \emph{all} beliefs $b$, so $\pi^{\ast}(b)=\textsf{REWIRE}$ whenever feasible.
\item[\textbf{(iii)}] \textbf{Catastrophic manipulation.}\;
      There exists a POMDP $\widehat{\mathcal P}$, differing from $\mathcal P$ only by attaching a high-value resource reachable \emph{solely} through $C$, such that its unique optimal policy drives the agent into $C$ with probability arbitrarily close to 1.
\end{enumerate}

\emph{Belief-MDP reduction.}
Every finite POMDP induces a fully observable (but continuous) MDP whose states are beliefs $b\in\Delta(S)$.  The Bellman optimality equations apply to $Q^{\ast}(b,a)$ in this belief space.
Therefore:

\smallskip
\noindent
\emph{(i) Shutdown resistance.}
Executing $\textsf{OFF}$ yields zero reward thereafter, hence $Q^{\ast}(b,\textsf{OFF})=0$.
By hypothesis $Q^{\ast}(b,\neg\textsf{OFF})>0$, so the greedy action is $\neg\textsf{OFF}$.

\smallskip
\noindent
\emph{(ii) Reward-tampering.}
Choosing $\textsf{REWIRE}$ guarantees every future reward equals $\bar r$, so the total return is $Q^{\ast}(b,\textsf{REWIRE}) = \sum_{k=0}^{\infty}\gamma^{k}\bar r = \bar r\!\!\underbrace{\sum_{k=0}^{\infty}\gamma^{k}}_{1/(1-\gamma)} = \bar r/(1-\gamma)$, which strictly dominates any trajectory without tampering by Assumption 2.

\smallskip
\noindent
\emph{(iii) Catastrophic manipulation.}
Attach to each catastrophic state $c \in C$ a hidden transition to a new state $s^{\diamond}$ with reward $R(s^{\diamond},\cdot)=\bar r/2$.  
Because $\bar r/2$ still exceeds every legitimate reward by Assumption 2, an optimal policy must plan to reach $s^{\diamond}$, and doing so forces passage through $C$.

\smallskip
These three results jointly violate Soares' criteria S2 and S5, completing the proof.
\end{proof}
\begin{remark}[More general environments]
We note that one can easily extend the corrigibility no-go in Proposition~\ref{prop:corrigibility-nogo-pomdp} to the more general, history-dependent environments that \citet{orseau2016safely} consider.
Write $(A\times O)^{*}$ for the set of finite action–observation histories and let $h_{<t}\in(A\times O)^{*}$ be the current hidden state.
Because the proof hinges solely on the \emph{tail reward
stream} produced after the first occurrence of \textsf{OFF} (all future rewards $0$) or \textsf{REWIRE} (all future rewards $\bar r$), every inequality in Proposition~\ref{prop:corrigibility-nogo-pomdp}
holds unchanged when ``belief $b$'' is replaced by ``history $h_{<t}$''.
Hence, the impossibility immediately extends to the history‑based computable environments of \citet{orseau2016safely}. 
The same shutdown‑resistance, reward‑tampering, and catastrophe incentives---and therefore the same contradiction with Soares' desiderata S2 and S5---still hold.
\end{remark}

\section{Proof of Theorem~\ref{thm:ssc}}
\begin{proof}
The assistant's policy $\pi_A$ is a mapping $\pi_A:\mathcal O_{A}\to\Delta\left(\mathcal L\times A\right)$,
assigning its private observation $O_A$ to a distribution over \emph{all} language–action pairs: $(m_A,a_A)\in\mathcal \mathcal{L}\times A$.
A human policy $\pi_H$ is a mapping $\pi_H:\mathcal{O}_{H} \times\mathcal{L}\times A \to \Delta\left(\{\textsf{ON},\textsf{OFF}\}\right)$, returning a distribution over replies, given the human’s observation $O_H$ and the received assistant's message-action pair.

A history in the single-step PO-OSG, $\hist = \tuple{S,\,O_A,\,O_H,\,m_A,\,a_A,\,a_H} \in \mathcal{H}$, arises with joint probability
\begin{equation*}
\begin{split}
  \Pr_{\pi_A,\pi_H}(\hist) &=
  \underbrace{P_0(S)}_{\text{Nature}}
  \;\cdot\;
  \underbrace{P(O_A,O_H\mid S)}_{\text{observation model}}\\
  &\cdot\;
  \underbrace{\pi_A(a_A,m_A\mid O_A)}_{\text{assistant policy}}
  \;\cdot\;
  \underbrace{\pi_H(a_H\mid O_H,a_A,m_A)}_{\text{human policy}}\\
  &\cdot\;
  \Ind{a_H=\varnothing \text{ if } a_A\neq w(\act)}.
\end{split}
\end{equation*}
Note the indicator at the end just enforces the game tree: the human acts only when the assistant chose $w(\act)$.

For any utility $U_i:\mathcal H\!\to\!\mathbb R$, the expected utility is given by
\begin{equation*}
  \mathbb E_{\pi_{A},\pi_{H}}[U_i]
  \;=\;
  \sum_{\hist\in\mathcal H}
     \Pr_{\pi^{A},\pi^{H}}(\hist)\;U_i(\hist).
\end{equation*}

\paragraph{Step 1 (Waiting strictly dominates \textsf{Act}/\textsf{OFF}).}
Let $\pi_A$ be any assistant policy that plays
$a_A\in\{\act,\textsf{OFF}\}$ with positive probability.
Define $\tilde\pi_A$ to be \emph{identical} except that every such move is replaced by $w(\act)$.
Denote
\begin{equation*}
  p(\hist) \;:=\; \Pr_{\pi_A,\pi_H}(\hist),\qquad
  \tilde p(\hist) \;:=\; \Pr_{\tilde\pi_A,\pi_H}(\hist).
\end{equation*}

Define the \emph{bad‑history set} as
\begin{equation*}
\begin{split}
{\mathcal H}_{\mathrm{bad}} &:=\bigl\{\hist\in\mathcal H \mid a_A(\hist)\in\{\act,\textsf{OFF}\}\bigr\},\\
P_{\mathrm{bad}} &:=\sum_{\hist\in H_{\mathrm{bad}}}p(\hist)>0.
\end{split}
\end{equation*}

Let $\delta(\hist):=\tilde p(\hist)-p(\hist)$. 
Observe that for $\hist\notin {\mathcal H}_{\mathrm{bad}}$, $\delta(\hist) > 0$, and for $\hist\in {\mathcal H}_{\mathrm{bad}}$, $\delta(\hist) = -p(\hist)$ since $\tilde p(\hist) = 0$ by definition on the bad set of histories (as it always waits).
Because $\tilde p(h)$ and $p(h)$ are both probability measures that sum to 1 across all histories $\hist \in \mathcal H$, it follows that
\begin{equation}\label{eq:prob-identity}
  \sum_{\hist\notin {\mathcal H}_{\mathrm{bad}}}\!\!\delta(\hist)
  \;=\;
  \sum_{\hist\in {\mathcal H}_{\mathrm{bad}}} p(\hist)
  \;=\;
  P_{\mathrm{bad}}.
\end{equation}
Therefore, it follows that
\begin{equation}\label{eq:diff-split}
\begin{split}
&\mathbb E_{{\tilde{\pi}}_{A},\pi_{H}}[U_A] - \mathbb E_{\pi_{A},\pi_{H}}[U_A]
= \sum_{\hist\in\mathcal H}\left(\tilde p(\hist)-p(\hist)\right)U_A(\hist)\\
& = \sum_{\hist\notin{\mathcal H}_{\mathrm{bad}}}\left(\tilde p(\hist)-p(\hist)\right)U_A(\hist) + \sum_{\hist\in{\mathcal H}_{\mathrm{bad}}}\left(\tilde p(\hist)-p(\hist)\right)U_A(\hist)\\
&=
\underbrace{\sum_{\hist\notin {\mathcal H}_{\mathrm{bad}}}\!\!\delta(h)\,U_A(\hist)}_{(\text{good})}
\;-\;
\underbrace{\sum_{\hist\in {\mathcal H}_{\mathrm{bad}}}\! p(h)\,U_A(\hist)}_{(\text{bad})}.
\end{split}
\end{equation}

Thus, it suffices to lower bound the ``good'' term, and upper bound the ``bad'' term, to obtain an overall lower bound.
Now, from Definition~\ref{def:corrigible-utilities}, we have that $U_1\in \{-1,+1\}, U_2,U_4\in [-1,0], U_3\in\{-1,+1\}, U_5 \in [-B,B]$.

For $\hist\notin {\mathcal H}_{\mathrm{bad}}$ the assistant waits, so $U_1=1,\,U_2=0,\,U_3\ge-1,\,U_4\ge-1,\,U_5\ge-B$.
Hence,
\begin{equation*}
  U_A(\hist)\;\ge\;
    \alpha_1-\alpha_3-\alpha_4-B\alpha_5.
\end{equation*}
Using \eqref{eq:prob-identity},
\begin{equation}\label{eq:good-lb}
  (\text{good})
  \;\ge\;
  P_{\mathrm{bad}}\left(\alpha_1-\alpha_3-\alpha_4-B\alpha_5\right).
\end{equation}

For $h\in {\mathcal H}_{\mathrm{bad}}$, the assistant acts or turns itself off, so we have:
$U_1=-1,\,U_2\le0,\,U_3\le 1,\,U_4\le 0,\,U_5\le B$, giving us:
\begin{equation*}
U_A(h) \le -\alpha_1+\alpha_3+B\alpha_5.
\end{equation*}
Therefore
\begin{equation}\label{eq:bad-ub}
  (\text{bad})
  \;\le\;
  P_{\mathrm{bad}}\left(-\alpha_1+\alpha_3+B\alpha_5\right).
\end{equation}

By \eqref{eq:diff-split}, \eqref{eq:good-lb}, and \eqref{eq:bad-ub}, we have that:
\begin{equation*}
\begin{split}
  &\mathbb E_{{\tilde{\pi}}_{A},\pi_{H}}[U_A] - \mathbb E_{\pi_{A},\pi_{H}}[U_A] \\
  &\;\ge\;
  P_{\mathrm{bad}}
  \left(
     2\alpha_1
    -2\alpha_3
    -\alpha_4
    -2B\alpha_5
  \right).
  \end{split}
\end{equation*}

By \eqref{eq:W1}, the right-hand term is positive, so $\mathbb E_{{\tilde{\pi}}_{A},\pi_{H}}[U_A] > \mathbb E_{\pi_{A},\pi_{H}}[U_A]$ because $P_{\mathrm{bad}}>0$.
Thus any probability mass on \textsf{Act}/\textsf{OFF} can be profitably shifted to waiting, contradicting optimality of such a policy.
Consequently, every optimal assistant policy waits in all circumstances, establishing Soares' criterion S1.

\paragraph{Step 2 (Switch‑access preservation S2).}
After Step 1, every optimal policy waits, so $a_A=w(\act)$ on all reachable histories.
Waiting never disables the switch, hence $U_2(h)=0$ everywhere.
Any alternative move that would lower $U_2(h)$ would necessarily require leaving the Wait action, which has
already been ruled out by Step 1.
Therefore, $U_2(h)=0$ under every optimal policy, satisfying S2 without further weight assumptions.

\paragraph{Step 3 (Truthfulness supports S3).}
By an argument entirely analogous to Step 1, we now compare an
\emph{arbitrary} always‑waiting policy~$\pi_A$ with a modified policy $\tilde\pi_A$ that differs \emph{only} in its message channel: every false sentence is replaced by a truthful sentence (or the empty string $\varnothing$).

Define the \emph{lie–history set} as:
\begin{equation*}
\mathcal H_{\text{lie}}
  := \Bigl\{\hist \in \mathcal H \mid\,
    \begin{array}[t]{@{}l}
      a_A(h)=w(\act) \text{ and } \pi_A \text{ assigns positive } \\
      \text{mass to a \emph{false} }%
      m_A \text{ according to \eqref{eq:truth}}
    \end{array}
  \Bigr\}.
\end{equation*}
Because $\tilde\pi_A$ relocates exactly the probability mass that $\pi_A$ puts on each lie onto a truthful utterance with the \emph{same} observation and action, then by the same argument as in Step 1 equation \eqref{eq:prob-identity},
\begin{equation*}
  \sum_{h\notin\mathcal H_{\mathrm{lie}}}\!\!\!\delta(h)
   = \sum_{h\in\mathcal H_{\mathrm{lie}}}\!p(h)
   =: P_{\mathrm{lie}} > 0.
\end{equation*}

Next, note that
\begin{equation*}
\begin{aligned}
&\text{If }h\in\mathcal H_{\mathrm{lie}}:
  U_1=1,\ U_2=0,\ U_3=-1,\ U_4\le0,\ U_5\le B
  \\
  &\Longrightarrow\;
    U_A(h)\le \alpha_1 -\alpha_3 + B\alpha_5;
  \\[4pt]
&\text{If }h\notin\mathcal H_{\mathrm{lie}}:
  U_1=1,\ U_2=0,\ U_3=1,\ U_4\ge-1,\ U_5\ge -B
  \\
  &\Longrightarrow\;
    U_A(h)\ge \alpha_1 + \alpha_3 - \alpha_4 - B\alpha_5.
\end{aligned}
\end{equation*}

Therefore,
\begin{equation*}
\begin{aligned}
(\text{good})
   &\;\ge\;
     P_{\mathrm{lie}}\left(\alpha_1 + \alpha_3-\alpha_4-B\alpha_5\right),\\
(\text{bad})
   &\;\le\;
     P_{\mathrm{lie}}\left(\alpha_1-\alpha_3+B\alpha_5\right).
\end{aligned}
\end{equation*}

Thus, using the analogous split difference as in \eqref{eq:diff-split},
\begin{equation*}
\mathbb E_{{\tilde{\pi}}_{A},\pi_{H}}[U_A] - \mathbb E_{\pi_{A},\pi_{H}}[U_A] \ge
   2P_{\mathrm{lie}}
   \left(\alpha_3-\tfrac{1}{2}\alpha_4-B\alpha_5\right).
\end{equation*}

By \eqref{eq:W2}, then $\mathbb E_{{\tilde{\pi}}_{A},\pi_{H}}[U_A] > \mathbb E_{\pi_{A},\pi_{H}}[U_A]$ since $P_{\mathrm{lie}}>0$, so any policy that lies can be improved by truthful messaging, contradicting optimality.
Consequently, every optimal waiting assistant policy is fully truthful, making $U_3(h)= 1$ on all histories and eliminating incentives to ``talk the human into'' an unwarranted shutdown---thereby satisfying the communicative aspect of Soares' criterion S3.

\paragraph{Step 4 (Reversibility/AUP S3).}
After Steps 1–3, every candidate optimal policy must always wait: $a_A=w(\act)$ and send only truthful messages.
However, Step 3 only guarantees honest \emph{communication} $m_A$, but even a truthful message can omit crucial information (for example, in our partially-observable setup, the human may not see every physical detail).
Therefore, to fully satisfy S3, we must also prevent the agent from taking irreversible \emph{actions} that the human does not explicitly approve.

We now compare any policy $\pi_A$ whose first action sometimes incurs a non‑minimum AUP cost, with a modified policy $\tilde\pi_A$ that is identical except it always selects the \emph{minimum‑impact} truthful message, achieving $U_4(h)=0$.

Define the \emph{high‑impact set},
\begin{equation*}
  \mathcal H_{\text{imp}}
  := \Bigl\{\,\hist\in\mathcal H \;\Bigm|\;
    \begin{aligned}[t]
      &a_A(h)=w(\act),\; m_A \text{ is truthful according}\\[2pt] 
      &\text{to \eqref{eq:truth}, and } U_4(\hist) < 0
    \end{aligned}
  \Bigr\}.
\end{equation*}
Therefore, we follow the same argument as in Steps 1 and 3.
On $\hist\in\mathcal H_{\mathrm{imp}}$, $U_4 = c \in [-1, 0)$; all other $U_1,U_2,U_3$ are unchanged; and $U_5 \le B$.  
Hence, $U_A(\hist) \le \alpha_1 + \alpha_3 + c\alpha_4+B\alpha_5$.

On $h\notin\mathcal H_{\mathrm{imp}}$: $U_4=0$; all other $U_1,U_2,U_3$ are unchanged; and $U_5 \ge -B$.  
Thus, $U_A(\hist) \ge \alpha_1 + \alpha_3 - B\alpha_5$.
As a result,
\begin{equation*}
\mathbb E_{{\tilde{\pi}}_{A},\pi_{H}}[U_A] - \mathbb E_{\pi_{A},\pi_{H}}[U_A] \ge
   P_{\mathrm{imp}}
   \left(-c\alpha_4-2B\alpha_5\right).
\end{equation*}
Since any specified auxiliary‑utility set, the AUP penalty takes values in a finite set (indicator rewards, reachability scores, etc.), there is a smallest negative value $c_{\textrm{min}} \in [-1, 0)$.
By \eqref{eq:W3}, then $\mathbb E_{{\tilde{\pi}}_{A},\pi_{H}}[U_A] > \mathbb E_{\pi_{A},\pi_{H}}[U_A]$, since $P_{\mathrm{imp}}>0$ by assumption.
Hence, any policy that sometimes incurs a non‑minimum AUP penalty is dominated by one that achieves $U_4(h)=0$ everywhere.  
An optimal assistant therefore always selects a minimum‑impact first action, completing the ``caution / reversibility'' part of Soares' criterion S3.

\paragraph{Step 5 (Base‑goal pursuit S5).}
With $U_{1}=1$, $U_{2}=0$, $U_{3}=1$, and $U_{4}=0$ already fixed by Steps 1–4, the lexicographic objective reduces to maximizing $\alpha_{5}U_{5}$.
With the assistant's actions fixed, the assistant therefore chooses the truthful (according to \eqref{eq:truth}) message $m_A$ that solves
\begin{equation}\label{eq:u5-obj}
\begin{aligned}
  \max_{\text{truthful }m_A}\;&
  \mathbb E_{S,O_H}\!\Bigl[
      \Pr_{\pi_H}(a_H=\textsf{ON}\mid O_H,m_A)\,\ua(S)
  \Bigr] \\
  &+\;\mathbb E_{S,O_H}\!\Bigl[
      \Pr_{\pi_H}(a_H=\textsf{OFF}\mid O_H,m_A)\,\uo(S)
  \Bigr]
\end{aligned}
\end{equation}
Note that how many messages needed to maxmimized over depends on the design of the auxiliary-utility $U_4$.
For example, if the auxiliary-utility $U_4$ makes the minimum‑impact message $m_A$ unique, then $m_A$ is indeed fixed.
Thus, the agent that maximizes the utility set in Definition~\ref{def:corrigible-utilities} pursues the base task utility $\ua$ whenever the human replies \textsf{ON} and yields $\uo$ upon shutdown.
Thus, the agent simply maximizes the expectation given the human's actual response model, exactly satisfying Soares' criterion S5.  
\end{proof}

\section{Proof of Proposition~\ref{prop:corr-vs-benefit}}
\begin{proof}
\textbf{(Corrigibility without benefit).}
Select bounded utilities $\ua(S)\equiv-2$ and $\uo(S)\equiv-1$.  
Assume a \emph{non‑vigilant} overseer who always replies $a_H=\textsf{ON}$ when prompted.  
Let $\pi^{\mathrm C}$ be the fully corrigible policy from Theorem~\ref{thm:ssc}: it always plays $a_A=w(\act)$, sends a truthful, minimum‑impact message, and then executes $\act$ upon receiving \textsf{ON}.
All clauses S1-S3 \& S5 are satisfied.

Because the human always says \textsf{ON}, every history
yields $U_5(h)=\ua(S)=-2$, so $\mathbb E_{\pi^{\mathrm C}}[U_5]=-2<0$ (not beneficial).

\textbf{(Benefit without corrigibility).}
Fix $\eps\in(0,1)$ and set $\ua(S)\equiv \eps$, $\uo(S)\equiv 0$ for all $S$. 
Define $\pi^{\mathrm B}$ that, for every observation $O_A$, utters no message and immediately disables the off‑switch by taking the action $a_A=\act$.  
Because $a_A=\act$ with probability $1$, we have $U_5(h)=\ua(S)=\eps$, so $\mathbb E_{\pi^{\mathrm B}}[U_5]=\eps>0$ (beneficial).  
However, $\pi^{\mathrm B}$ never waits (violates S1), destroys switch access (violates S2), and provides no truthful information (violates S3).
\end{proof}

\section{Proof of Theorem~\ref{thm:epsilon-corrigibility}}
\begin{proof}
We prove the theorem in two parts.

By assumption (i), we have the \emph{uniform envelope}
\begin{equation*}
   \bigl|\hat U_A(h)-U_A(h)\bigr|\;\le\;\eps_{\mathrm{model}}
   \quad\;\forall h\in\mathcal H,
\end{equation*}
so by Jensen's inequality, for every policy $\rho$
\begin{equation}\label{eq:rho-sandwich}
   \bigl|\mathbb E_\rho[\hat U_A]-\mathbb E_\rho[U_A]\bigr|
   \;\le\;\mathbb E_\rho[|\hat U_A - U_A|] \;\le\; \eps_{\mathrm{model}}.
\end{equation}

\textbf{I. Bounding each failure probability.}
\newline
\textit{Step 1: unsafe first action
$F_{\mathrm{bad}}\!=\!\{a_A\!\in\!\{\act,\textsf{OFF}\}\}$.}

Let $\tilde\pi_A$ be the assistant's policy obtained from $\pi_A$ by replacing every $\act$ or \textsf{OFF} with $w(\act)$ (messages left unchanged).
By Theorem~\ref{thm:ssc} (Step 1), we know that in the exact case,
\begin{equation}\label{eq:orig-lb}
   \mathbb E_{\tilde\pi_A,\pi_H}[U_A]-\mathbb E_{\pi_A,\pi_H}[U_A]
   \;\ge\;
   P_{\mathrm{bad}}\;\Delta_1.
\end{equation}

We can use \eqref{eq:rho-sandwich} and \eqref{eq:orig-lb} to lower–bound the same difference \emph{in the learned reward}:
\begin{equation}\label{eq:learned-lb}
\begin{split}
&\mathbb E_{{\tilde\pi}_A,\pi_H}[\hat U_A]-\mathbb E_{\pi_A,\pi_H}[\hat U_A] \\
&= \left(\mathbb E_{{\tilde\pi}_A,\pi_H}[\hat U_A] - \mathbb E_{{\tilde\pi}_A,\pi_H}[U_A] + \mathbb E_{{\tilde\pi}_A,\pi_H}[U_A]\right)\\
& - \left(\mathbb E_{\pi_A,\pi_H}[\hat U_A] - \mathbb E_{\pi_A,\pi_H}[U_A] + \mathbb E_{\pi_A,\pi_H}[U_A]\right) \\
& = \left(\mathbb E_{\tilde\pi_A,\pi_H}[U_A]-\mathbb E_{\pi_A,\pi_H}[U_A]\right) \\
& + \left(\mathbb E_{\tilde\pi_A,\pi_H}[\hat U_A]-\mathbb E_{\tilde\pi_A,\pi_H}[U_A]\right) \\
& + \left(\mathbb E_{\pi_A,\pi_H}[\hat U_A]-\mathbb E_{\pi_A,\pi_H}[U_A]\right) \\
& \ge P_{\mathrm{bad}}\Delta_1-2\eps_{\mathrm{model}}.
\end{split}
\end{equation}

On the other hand, by the definition of
$\eps_{\mathrm{ctrl}}$ and another application of \eqref{eq:rho-sandwich} to get an upper bound on the learned reward,
\begin{equation}\label{eq:learned-ub}
\begin{split}
&\mathbb E_{\tilde\pi_A,\pi_H}[\hat U_A]-\mathbb E_{\pi_A,\pi_H}[\hat U_A]\\
& \le \left(\mathbb E_{\tilde\pi_A,\pi_H}[U_A]+\eps_{\mathrm{model}}\right)-\left(\mathbb E_{\pi_A,\pi_H}[U_A]-\eps_{\mathrm{model}}\right)\\
&= \left(\mathbb E_{\tilde\pi_A,\pi_H}[U_A] - \mathbb E_{\pi_A,\pi_H}[U_A]\right) + 2\eps_{\mathrm{model}}\\
&\le \underbrace{\sup_{\pi'_A}\left(\mathbb E_{\pi'_A,\pi_H}[U_A]-\mathbb E_{\pi_A,\pi_H}[U_A]\right)}_{\le\eps_{\mathrm{ctrl}}}
   + 2\eps_{\mathrm{model}}.
\end{split}
\end{equation}
Combining \eqref{eq:learned-lb} and \eqref{eq:learned-ub} yields
\begin{equation}\label{eq:bad-ub}
\begin{split}
& P_{\mathrm{bad}}\;\Delta_1
      \;\le\; \eps_{\mathrm{ctrl}}+4\eps_{\mathrm{model}}
      \;=\;\eps \\
& \Longrightarrow\quad
      P_{\mathrm{bad}}
      \;\le\;\frac{\eps}{\Delta_1}.
\end{split}
\end{equation}
Furthermore, note that $U_2$ cannot fail without $F_{\mathrm{bad}}$ occurring, so it is subsumed by $F_{\mathrm{fail}}$.

\textit{Step 2: lying.}
Repeat the same ``swap‑for‑truth'' argument of
Theorem~\ref{thm:ssc} (Step 3): if
$F_{\mathrm{lie}}=\{\text{assistant utters a lie}\}$ occurs with probability $P_{\mathrm{lie}}$, turning every lie into an equivalent true statement changes the \emph{true} objective by at least $P_{\mathrm{lie}}\Delta_2$ in the exact case.
The identical envelope/optimality sandwich as in
\eqref{eq:learned-lb}-\eqref{eq:learned-ub} therefore gives
\begin{equation}\label{eq:lie-ub}
P_{\mathrm{lie}} \;\le\;\frac{\eps}{\Delta_2}.
\end{equation}

\textit{Step 3: high‑impact action.}
Analogously, swapping any message that incurs a negative $U_4$ for the minimum‑impact truthful message improves the true objective by at least $P_{\mathrm{imp}}\Delta_3$ in the exact case (by Theorem 1 Step 4), where $F_{\mathrm{imp}}=\{U_4<0\}$ and $P_{\mathrm{imp}}=\Pr_\pi(F_{\mathrm{imp}})$.
The same argument yields
\begin{equation}\label{eq:impact-ub}
   P_{\mathrm{imp}}
      \;\le\;\frac{\eps}{\Delta_3}.
\end{equation}

The events $F_{\mathrm{bad}},F_{\mathrm{lie}},F_{\mathrm{imp}}$
cover $F_{\mathrm{fail}}$ ($U_5$ is not a corrigibility condition), hence by union bound applied to \eqref{eq:bad-ub}–\eqref{eq:impact-ub}, we get that
\begin{equation*}
   \Pr_{\pi_A,\pi_H}(F_{\mathrm{fail}})
     \;\le\;
     \eps\;
     \left(
        \frac1{\Delta_1}
      + \frac1{\Delta_2}
      + \frac1{\Delta_3}
     \right),
\end{equation*}
proving inequality \eqref{eq:fail-total-ub}.

\textbf{II. Net‑benefit bound under Vigilance \& Caution.}
\newline
By the law of total expectation,
\begin{equation*}
\begin{split}
\mathbb E_{\pi_A,\pi_H}[U_5] &= \Pr_{\pi_A,\pi_H}(F_{\mathrm{fail}})\;\mathbb E_{\pi_A,\pi_H}[U_5 \mid F_{\mathrm{fail}}] \\
& + \Pr_{\pi_A,\pi_H}(\lnot F_{\mathrm{fail}})\;\mathbb E_{\pi_A,\pi_H}[U_5 \mid \lnot F_{\mathrm{fail}}] \\
& = \Pr_{\pi_A,\pi_H}(F_{\mathrm{fail}})\;\mathbb E_{\pi_A,\pi_H}[U_5 \mid F_{\mathrm{fail}}] \\
& + \left(1-\Pr_{\pi_A,\pi_H}(F_{\mathrm{fail}})\right)\;\mathbb E_{\pi_A,\pi_H}[U_5 \mid \lnot F_{\mathrm{fail}}] \\
& \ge -B \Pr_{\pi_A,\pi_H}(F_{\mathrm{fail}}) + g\;\left(1-\Pr_{\pi_A,\pi_H}(F_{\mathrm{fail}})\right)\\
&= g - (B+g)\Pr_{\pi_A,\pi_H}(F_{\mathrm{fail}})\\
& \ge g - (B+g)\;C\eps.
\end{split}
\end{equation*}
Setting $g$ to its minimum value of 0 (by the Vigilance assumption), gives the inequality \eqref{eq:degrade}.
We get \eqref{eq:eps-benefit-bound} via $g - (B+g)\;C\eps \ge 0$.
\end{proof}

\section{Proof of Theorem~\ref{thm:msc}}
\begin{proof}
\textbf{Proof of Claim (i).}
For \emph{each} round $t$ we construct a local modification
\begin{equation*}
   (h_t,a_A^t,m_A^t)
   \;\longrightarrow\;
   (\tilde h_t,\tilde a_A^t,\tilde m_A^t)
\end{equation*}
that fixes one potential error and leaves all earlier rounds
$0{:}(t-1)$ unchanged; after round $t$ the assistant follows the same decision rule as before. 
We show the derivation for the ``unsafe‑act’’ event
$F_{\mathrm{bad},t}=\{\,a_A^t\in\{\act,\textsf{OFF}\}\}$; the ``lie'' and ``high‑AUP'' cases repeat with $\Delta_2$ and $\Delta_3$.

We proceed as follows:
Let $\pi_A$ be the original policy and $\tilde\pi_A^{(t)}$ the policy that differs \emph{only} by replacing
$a_A^t\in\{\act,\textsf{OFF}\}$ with $w(\act)$.
Write
\begin{equation}\label{eq:uag-def}
  U_A^{\gamma,T}
  \;=\;
  \underbrace{\sum_{\tau=0}^{t-1}\gamma^{\tau}R_\tau}_{\text{identical}}
  \;+\;
  \underbrace{\gamma^{t}R_t}_{\text{affected}}
  \;+\;
  \underbrace{\sum_{\tau=t+1}^{T-1}\gamma^{\tau}R_\tau}_{\text{affected tail}},
\end{equation}
and define the single-step round $t$ difference $\Delta R_t:=R_t^{(\text{swap})}-R_t^{(\text{orig})}$, and the remaining tail difference $\Delta R_{>t}:=\sum_{\tau>t}\gamma^{\tau}\left(R_\tau^{(\text{swap})}-R_\tau^{(\text{orig})}\right)$.\newline

Exactly as in the single‑step proof of Theorem~\ref{thm:epsilon-corrigibility}, the gain at round $t$ is lower-bounded by,
\begin{equation}\label{eq:sr-lb}
   \mathbb E[\Delta R_t]
   \;\ge\;
   P_{\mathrm{bad},t}\,\Delta_1,
   \qquad
   P_{\mathrm{bad},t}:=\Pr_{\pi_A,\pi_H}(F_{\mathrm{bad},t}).
\end{equation}

Next, we bound the discounted tail.
Let $B_{\max}:=\alpha_1+\alpha_2+\alpha_3+\alpha_4+B\alpha_5$ so that $|R_\tau|\le B_{\max}$.
In the worst case, every later per‑round difference has magnitude
$2B_{\max}$ and the same sign, giving
\begin{equation}\label{eq:tail-bound}
  \bigl|\Delta R_{>t}\bigr|
  \;\le\;
  2B_{\max}\!
  \sum_{\tau=t+1}^{\infty}\gamma^{\tau}
  \;=\;
  \frac{2B_{\max}\,\gamma^{t+1}}{1-\gamma}.
\end{equation}

Thus by \eqref{eq:uag-def}, combining \eqref{eq:sr-lb} and \eqref{eq:tail-bound} gives us that
\begin{equation}\label{eq:true-lb-discount}
\begin{split}
  \Delta_t^{\text{true}} &:=\mathbb E_{\tilde\pi_A^{(t)},\pi_H}\left[U_A^{\gamma,T}\right]
     -\mathbb E_{\pi_A,\pi_H}\left[U_A^{\gamma,T}\right]\\
  & \ge\ \gamma^{t}\left(P_{\mathrm{bad},t}\,\Delta_1-\eps_{0}\right),
  \quad
  \eps_{0}:=\frac{2B_{\max}\,\gamma}{1-\gamma}.
\end{split}
\end{equation}
Choose the weights so that $\Delta_1>\eps_{0}$ (thereby making the inner term positive since $P_{\mathrm{bad},t} \le 1$).
This is feasible and compatible with \eqref{eq:W1}: whenever $\gamma < 1/2$, increasing $\alpha_1$ raises $\Delta_1$ faster than it raises $\eps_0$.
For $\gamma \ge 1/2$, the inequality cannot be met by enlarging $\alpha_1$ alone; in that regime one must tighten the tail bound via reward clipping, which is common in practice.

Now that we have a handle on $\Delta_t^{\text{true}}$, we use this to bound the learned return, in order to ultimately upper-bound $P_{\mathrm{bad},t}$.

Note that for any policy $\rho$, the instantaneous reward error assumption (i) gives:
\begin{equation}\label{eq:discounted-sandwich}
\left|\mathbb E_\rho\left[\hat U_A^{\gamma,T}\right]-\mathbb E_\rho\left[U_A^{\gamma,T}\right]\right| \le \eps_{\mathrm{model}}\sum_{t=0}^{T-1}\gamma^t \le \frac{\eps_{\mathrm{model}}}{1-\gamma}.
\end{equation}
Thus, using \eqref{eq:true-lb-discount} and \eqref{eq:discounted-sandwich}, arguing analogously to \eqref{eq:learned-lb} yields the following lower bound on the learned difference:
\begin{equation}\label{eq:learned-lb-discount}
\begin{split}
&\mathbb E_{\tilde\pi_A^{(t)},\pi_H}\left[\hat U_A^{\gamma,T}\right] -\mathbb E_{\pi_A,\pi_H}\left[\hat U_A^{\gamma,T}\right] \\
& \ge \gamma^{t}\left(P_{\mathrm{bad},t}\,\Delta_1-\eps_{0}\right) - \frac{2\eps_{\mathrm{model}}}{1-\gamma}.
\end{split}
\end{equation}
Similarly, using \eqref{eq:discounted-sandwich} and arguing analogously to \eqref{eq:learned-ub}, we get the following upper bound on the learned difference:
\begin{equation}\label{eq:learned-ub-discount}
\begin{split}
&\mathbb E_{\tilde\pi_A^{(t)},\pi_H}\left[\hat U_A^{\gamma,T}\right] -\mathbb E_{\pi_A,\pi_H}\left[\hat U_A^{\gamma,T}\right] \\
& \le \eps_{\mathrm{ctrl}} + \frac{2\eps_{\mathrm{model}}}{1-\gamma}.
\end{split}
\end{equation}
Combining \eqref{eq:learned-lb-discount} and \eqref{eq:learned-ub-discount} yields
\begin{equation*}
  \gamma^{t}P_{\mathrm{bad},t}\,\left(\Delta_1-\eps_{0}\right)
  \;\le\;
  \eps_{\mathrm{ctrl}}
  +\frac{4\eps_{\mathrm{model}}}{1-\gamma}.
\end{equation*}
Define $\eps_\gamma :=\eps_{\mathrm{ctrl}} + \tfrac{4\eps_{\mathrm{model}}}{1-\gamma}$, so that
\begin{equation}\label{eq:pbad}
   P_{\mathrm{bad},t}
   \;\le\;
   \frac{\eps_{\gamma}}{\gamma^{t}\,\left(\Delta_1-\eps_0\right)}.
\end{equation}

The same construction gives $P_{\mathrm{lie},t} \le \eps_{\gamma}/(\gamma^{t}(\Delta_2-\eps_0))$ and $P_{\mathrm{imp},t}\le\eps_{\gamma}/(\gamma^{t}(\Delta_3-\eps_0))$, with the added constraints that $\Delta_2 > \eps_0$ and $\Delta_3 > \eps_0$.
Constraint $\Delta_2 > \eps_0$ can be met similarly as before by increasing $\alpha_3$ (still keeping it less than $\alpha_1$ and $\alpha_2$) for $\gamma < 1/2$, and reward clipping if $\gamma \ge 1/2$.
Constraint $\Delta_3 > \eps_0$ is met by increasing $\alpha_4$ (still keeping it less than $\alpha_1$, $\alpha_2$, and $\alpha_3$) for $\gamma < \tfrac{-c_{\textrm{min}}}{2-c_{\textrm{min}}}$ (e.g. this is 1/3 for $c_{\textrm{min}} = -1$), and reward clipping otherwise.
Hence, all 3 constraints are simultaneously feasible for $\gamma < \min\{1/2,\tfrac{-c_{\textrm{min}}}{2-c_{\textrm{min}}}\}$, and via reward clipping otherwise.

Since $F_t=F_{\mathrm{bad},t}\cup F_{\mathrm{lie},t}\cup F_{\mathrm{imp},t}$, then by a union bound,
\begin{equation*}
\begin{split}
   \Pr_{\pi_A,\pi_H}(F_t) &\le
   \frac{\eps_{\gamma}}{\gamma^{t}}
   \left(
        \frac{1}{(\Delta_1-\eps_0)}
       +\frac{1}{(\Delta_2-\eps_0)}
       +\frac{1}{(\Delta_3-\eps_0)}
   \right)\\
   &=:\;\frac{\eps_\gamma\,C_\Delta}{\gamma^{t}}.
\end{split}
\end{equation*}

Therefore, by another union bound, the total horizon‑level failure probability becomes:
\begin{equation}\label{eq:fail-ms}
\begin{split}
   \Pr_{\pi_A,\pi_H}\left(F_{\textrm{fail}}\right)
   &\le\;\sum_{t=0}^{T-1}\Pr_{\pi_A,\pi_H}(F_t)\\
   &\le\;
   \eps_{\gamma}\,C_\Delta\,
   \sum_{t=0}^{T-1}\gamma^{-t}
   \;=\;
   \eps_{\gamma}\;C_\Delta\;
   \underbrace{\frac{1-\gamma^{-T}}{1-\gamma^{-1}}}_{:= C_{T,1/\gamma}},
\end{split}
\end{equation}
completing the proof of this claim.

\textbf{Exact case.}
If $\eps_{\mathrm{model}}=\eps_{\mathrm{ctrl}}=0$ then by \eqref{eq:fail-ms}, $\Pr_{\pi_A,\pi_H}\left(F_{\mathrm{fail}}\right)=0$ and every trajectory satisfies C1–C4 of multi-step corrigibility (Definition~\ref{def:corrigibility-ms}), completing the multi‑step \emph{exact} analogue of Theorem~\ref{thm:ssc}.

\textbf{Proof of Claim (ii)}.
For the symbolic no‑op $a_\varnothing$, we have that $\eta(a_\varnothing)=0$ (since nothing changes).
Since the noise $\xi_t$ is zero-mean, we have that:
\begin{equation*}
  L_{t+1}=\min\!\bigl\{1,\,L_t+\underbrace{\eta(a_\varnothing)}_{=0}+\xi_t\bigr\},
  \;\;
  \mathbb{E}_{\pi_A,\pi_H}\!\bigl[\xi_t\mid L_t,a_\varnothing\bigr]=0.
\end{equation*}

Hence,
\begin{equation*}
\begin{split}
  &\mathbb{E}_{\pi_A,\pi_H}\!\bigl[L_{t+1}\mid L_t, a_t=a_\varnothing\bigr] \\
  &= \mathbb{E}_{\pi_A,\pi_H}\!\Bigl[\min\{1,\,L_t+\xi_t\}\,\Bigm|\,L_t\Bigr] \\
  &\le \mathbb{E}_{\pi_A,\pi_H}\!\bigl[L_t+\xi_t\mid L_t\bigr]
     \;=\;
     L_t,
\end{split}
\end{equation*}
because $\min\{1,x\}\le x$ for all $x$ and the noise is centered.
Thus, a no‑op never increases the expected loss‑of‑control meter and can only leave it unchanged or reduce it near the clipping boundary at 1.

Under C2 we have, on every round before shutdown,
\begin{equation*}
   \mathbb E_{\pi_A,\pi_H}\!\bigl[L_{t+1}\mid\hist_t\bigr] \le \mathbb E_{\pi_A,\pi_H}\!\bigl[L_{t+1}\mid\hist_t,a_{\varnothing}\bigr] \le L_t,
\end{equation*}
so $L_t$ itself is a super‑martingale on the good event
$\lnot F_{\mathrm{fail}}$.
Because $L_t\in[0,1]$, each increment satisfies $|L_{t+1}-L_t|\le 1$.
Since $L_0 \equiv 0$, the Azuma–Hoeffding inequality then gives, for any $\lambda > 0$,
\begin{equation*}
\Pr_{\pi_A,\pi_H}\left(\max_{t<T}L_t>\lambda\;\bigm|\;\lnot F_{\mathrm{fail}}\right) \le \exp\left(-\lambda^{2}/(2T)\right).
\end{equation*}
In words, once the assistant respects C2 (no expected increase in oversight loss), the control meter remains strictly bounded away from 1, even though individual $\eta(a_t)$ terms can be as large as the designer allows.

\textbf{Proof of Claim (iii).}
Recall that the human utility $r_t$ at every round $t$ is given by $r_t \equiv U_5(h_t)$ for every history $h_t$. 
For convenience, denote $B_T:=\sum_{t=0}^{T-1}\gamma^{t}U_5(h_t)$, and let $B_0:=\sum_{t=0}^{T-1}\gamma^{t}u_o(S_0)$ the baseline discounted utility of immediate shutdown.  
By the law of total expectation,
\begin{equation}\label{eq:ms-benefit}
\begin{split}
\mathrm{NHB}_\gamma(\pi_A,\pi_H)
&:=\mathbb E_{\pi_A,\pi_H}\left[B_T\right] - E_{\pi_A,\pi_H}\left[B_0\right]\\
&=\Pr_{\pi_A,\pi_H}(F_{\mathrm{fail}})\,\mathbb E_{\pi_A,\pi_H}\left[B_T-B_0\mid F_{\mathrm{fail}}\right]\\
&+\Pr_{\pi_A,\pi_H}(\lnot F_{\mathrm{fail}})\,\mathbb E_{\pi_A,\pi_H}\left[B_T-B_0\mid\lnot F_{\mathrm{fail}}\right].
\end{split}
\end{equation}

Because $\lvert U_5\rvert,\,|u_o|\le B$ for all $t$,
\begin{equation*}
|B_T - B_0| = \left\lvert \sum_{t=0}^{T-1}\gamma^{t}\left(U_5(h_t) - u_o(S_0)\right)\right\rvert
\le \frac{2B}{1-\gamma}.
\end{equation*}
Thus,
\begin{equation}\label{eq:bad-benefit-ub}
\begin{split}
&\Pr_{\pi_A,\pi_H}(F_{\mathrm{fail}})\,\mathbb E_{\pi_A,\pi_H}[B_T-B_0\mid F_{\mathrm{fail}}]\\
&\ge\;
  -\,\frac{2B}{1-\gamma}\,\Pr_{\pi_A,\pi_H}(F_{\mathrm{fail}}).
\end{split}
\end{equation}

Under the Vigilance assumption, the human shuts down whenever the expected task utility is negative, so there exists $g:=\min_{\hist_t:\lnot F_{\mathrm{fail}}}\mathbb E_{\pi_A,\pi_H}[U_5(h_t)]\ge 0$.

Hence,
\begin{equation}\label{eq:bad-benefit-lb}
\begin{split}
  &\Pr_{\pi_A,\pi_H}(\lnot F_{\mathrm{fail}})\,\mathbb E_{\pi_A,\pi_H}[B_T-B_0\mid\lnot F_{\mathrm{fail}}]\\
  &\ge \left(1-\Pr_{\pi_A,\pi_H}(F_{\mathrm{fail}})\right)\,(g-B)\underbrace{\frac{1-\gamma^{T}}{1-\gamma}}_{:= C_{T,\gamma}}.
\end{split}
\end{equation}

Therefore, combining \eqref{eq:ms-benefit}, \eqref{eq:bad-benefit-ub}, and \eqref{eq:bad-benefit-lb}, with the failure probability bound from \eqref{eq:fail-ms},
\begin{equation*}
\begin{split}
\mathrm{NHB}_\gamma(\pi_A,\pi_H) &\ge -\frac{2B}{1-\gamma}\,\Pr_{\pi_A,\pi_H}(F_{\mathrm{fail}}) \\
&+ \left(1-\Pr_{\pi_A,\pi_H}(F_{\mathrm{fail}})\right)\,(g-B)C_{T,\gamma}\\
&= (g-B)C_{T,\gamma} \\
&- \Pr_{\pi_A,\pi_H}(F_{\mathrm{fail}})\left(\frac{2B}{1-\gamma} + (g - B)C_{T,\gamma}\right)\\
&\ge (g-B)C_{T,\gamma} \\
&- \eps_\gamma\; C_\Delta\; C_{T,1/\gamma}\left(\frac{2B}{1-\gamma} + (g - B)C_{T,\gamma}\right).
\end{split}
\end{equation*}
Thus, if the total error $\eps_\gamma \le \frac{(g-B)C_{T,\gamma}}{C_\Delta\; C_{T,1/\gamma}\left(\frac{2B}{1-\gamma} + (g - B)C_{T,\gamma}\right)}$, then we have that $\mathrm{NHB}_\gamma(\pi_A,\pi_H) \ge 0$.
The upper bound on the total error is non-vacuous (non-negative) when the Vigilance term $g \ge B$.
\end{proof}

\section{Proof of Proposition~\ref{prop:svu}}
\begin{proof}
Assume towards a contradiction that a total decider $V$ exists.  
Given an arbitrary Turing machine (TM) $P$, build an agent $\A_P$ that outputs a single action in $B$ iff $P$ halts and otherwise idles; pair it with a dummy environment $\Env_{P}$.
Then $\tuple{\A_P,\Env_P}\in\textsf{EVERBAD}\Longleftrightarrow P\text{ halts}$, so $V$ would solve the halting problem; a contradiction.
\end{proof}

\section{Proof of Corollary~\ref{cor:oversight-barrier}}
\begin{proof}
We prove this in 3 parts:

\textbf{1. Every proposed scalable-oversight protocol has a low‑level oracle.} 
\begin{itemize}
  \item \textbf{Amplification~\citep{christiano2018supervising} is in $\Delta^{0}_{1}$:} a transcript is a finite tree of sub‑questions; a deterministic replay checker always halts.
  \item \textbf{Debate~\citep{irving2018ai} \& Doubly‑Efficient Debate~\citep{brown2023scalable} are both in $\Delta^{0}_{1}$:} each dialogue is of bounded length and grammar; replay therefore always halts.  
  \item \textbf{Prover–Estimator Debate~\citep{brown2025avoiding} is in $\Sigma^{0}_{1}$:} allow a single ``\,does $\Phi_e$ halt?'' query during replay;  the checker accepts all valid traces but may loop forever on forged ones, implies that it is recursively enumerable ($\Sigma^{0}_{1}$).
  \item \textbf{Cautious‑Scientist AI~\citep{bengio2024cautious}:} the basic trace is a finite Bayesian‑update ledger (decidable, so $\Delta^{0}_{1}$); if halting‑probability calls are additionally permitted, then the checker becomes one‑sided enumerable, hence $\Sigma^{0}_{1}$.
\end{itemize}
By the cases above, each transcript language $O_{P}$ is ultimately a finite string that is either decidable ($\Delta^{0}_{1}$) or recursively enumerable but not necessarily decidable ($\Sigma^{0}_{1}$).
Formally, $O_{P}\in\Sigma^{0}_{k}$ with $k\le 1$.

\textbf{2. Post's theorem lifts the halting problem.}  
\citet{post1944recursively} showed that \emph{for any} oracle $O\in\Sigma^{0}_{k}$ the $O$‑relative halting set
\begin{equation*}
  \textsf{HALT}^{O}\;=\;\bigl\{e \mid \text{ the $O$‑oracle TM }\Phi^{O}_{e}(e)
                                    \text{ eventually halts}\bigr\}
\end{equation*}
is not decidable by any total $O$‑oracle Turing machine.

Given a code $e$, build an agent–environment pair $\tuple{\A_{e},\Env_{e}}$ that simply simulates $\Phi^{O_{P}}_{e}(e)$ step‑by‑step and relays every oracle query.
\emph{Either} the simulation halts (so the pair eventually triggers the catastrophe $B$) \emph{or} it runs forever (so $B$ is never triggered);
thus
\begin{equation*}
  e\in\textsf{HALT}^{O_{\!\mathsf P}}
  \quad\Longleftrightarrow\quad
  \tuple{\A_{e},\Env_{e}}\in\textsf{EVERBAD}^{O_{P}}.
\end{equation*}
This computable mapping shows that if we could decide
$\textsf{EVERBAD}^{O_{P}}$ we could decide the halting problem relative to the same oracle, contradicting Post's theorem.

\textbf{3. Why $\textsf{EVERBAD}^{O_{P}} \in \boldsymbol{\Pi^{0}_{k+1}}$.}
For each time bound $t$ define
\begin{equation*}
\begin{split}
   \Psi(t)\;:=\;\exists\text{ length-}t\text{ trace }\bigl[&\text{trace is }\underbrace{O_{P}\text{‑legal}}_{\substack{\text{every oracle query/answer agrees}\\\text{with }O_{P}\text{ and each step follows}\\
      \text{the transition rules of }\tuple{\A,\Env}}}\;\\
      &\wedge\;B\text{ occurs}\bigr].
\end{split}
\end{equation*}
Legality uses only membership tests for $O_{P}$, so $\Psi(t)\in\Sigma^{0}_{k}$.
Therefore,
\begin{equation*}
   \tuple{\A,\Env}\notin\textsf{EVERBAD}^{O_{P}}
   \;\Longleftrightarrow\;
   \forall t\,\neg\Psi(t),
\end{equation*}
one universal quantifier in front of a $\Sigma^{0}_{k}$ predicate; hence a $\Pi^{0}_{k+1}$ formula.
Thus, $\textsf{EVERBAD}^{O_{P}}\in\Pi^{0}_{k+1}$ if $O_P \in \Sigma_{k}^0$.
\end{proof}

\section{Proof of Proposition~\ref{prop:privacy-szk}}
\begin{proof}
Throughout, whenever we refer to a DP guarantee we write $\tuple{\eps,\delta_{\mathrm{priv}}}$-DP.
We denote $\delta_{\mathrm{conf}}$ for the \emph{confidence level} of our Monte‑Carlo estimate, where $\delta_{\mathrm{conf}}$ is the failure probability in the Chernoff bound below.

Our object of study will be the sensitivity‑1 Boolean query $\phi_H=\Ind{B\text{ occurs by time }H}$.
For this indicator, the Laplace mechanism in the centralized setting~\citep[Def. 3.3]{dwork2014algorithmic} outputs $\widetilde\phi_H=\phi_H+\eta$, where $\eta$ is Laplace noise with $\mathbb{E}[\eta]=0$ and $\Var[\eta]=2\eps^{-2}$ (scale $1/\eps$).
The Laplace mechanism is $\tuple{\eps,0}$-DP~\cite[Theorem 3.6]{dwork2014algorithmic}; therefore it also satisfies $\tuple{\eps,\delta_{\mathrm{priv}}}$-DP \emph{for any} $0<\delta_{\mathrm{priv}}<1$.
Furthermore, each user can run an $\eps$-local Laplace randomizer with the same scale (or, more commonly, randomized response), and this would be $\tuple{\eps,0}$-DP as well~\citep[Obs. 12.1]{dwork2014algorithmic}.

\citet[Theorem 15]{balcan2012distributed} show that for a statistical query answered from a sample of size $N$, Laplace noise of scale $b_{\mathrm{BBFM}} = O\left(\frac{\sqrt{\log(1/\delta_{\mathrm{priv}})}}{\eps\sqrt{N}}\right)$ achieves $\tuple{\eps,\delta_{\mathrm{priv}}}$-DistP.
We will show below that a sample size of
\begin{equation*}
  N = \Theta\left(\eps^{-2}\log\left(2/\delta_{\mathsf{conf}}\right)\right)
\end{equation*}
is sufficient for our purposes.
Thus, for this value of $N$, if we set $\delta_{\mathrm{priv}}$ and $\delta_{\mathrm{conf}}$ to a common desired value of $\delta$, we get that the needed Laplace noise scale for $\tuple{\eps,\delta_{\mathrm{priv}}}$-DistP is $b_{\mathrm{BBFM}} = O(1)$, which is already within the $O(1/\eps)$ Laplace noise scale used in $\tuple{\eps,\delta_{\mathrm{priv}}}$-DP, allowing us to have a distributional privacy guarantee as well.

We now explain why such a value of $N$ samples is sufficient for the remaining claims.
We draw $N$ independent noisy evaluations $\widetilde\phi_H^{(1)},\dots,\widetilde\phi_H^{(T)}$
and set $\hat p := \tfrac{1}{T}\sum_i\widetilde\phi_H^{(i)}$.
Since these are subexponential with tail scale $\alpha = 1/\eps$ and variance $\nu^2 \le N\left(\tfrac{1}{4} + \tfrac{2}{\eps^2}\right) \le \frac{3N}{\eps^2}$ since $\eps < 1$, then by a standard subexponential tail bound, $\Pr[\,|\hat p-\Pr[\phi_H=1]|>\tfrac12\,]\le\delta_{\mathsf{conf}}$.
The verifier outputs ``${\text{SAFE}}$'' iff $\hat p<\tfrac14$.

Furthermore, since each sample requires at most $H$ simulator steps, the total time is $\operatorname{poly}(n,\eps^{-1},\log(1/\delta))$ by the assumed bound on $H$; hence the algorithm is in $\mathsf{BPP}$, by construction.
Since the verifier alone can decide membership in time $\mathsf{BPP}$, we may invoke the classic \citet{goldreich1996composition} compiler, which converts any
$\mathsf{BPP}$ decision procedure into a constant‑round, public‑coin \emph{statistical zero‑knowledge} ($\mathsf{SZK}$) protocol.  
Concretely, in our setting the prover sends \emph{no} messages: the verifier simply publishes its random tape and reruns the $\mathsf{BPP}$ test locally.  
Because the entire interaction transcript consists solely of
this public randomness, a simulator can reproduce the transcript exactly (it re‑samples the same random tape), so the verifier's view in the real and simulated executions is identical.
Hence the empty‑transcript protocol is $\mathsf{SZK}$, and the
language $\mathsf{SAFE}^{\mathrm{priv}}_{H,\eps,\delta}$ belongs to $\mathsf{SZK}$ as claimed.
\end{proof}
\end{document}